\documentclass{article}

\usepackage[preprint]{neurips_2026} 

\usepackage{amsmath, amssymb, amsthm}
\usepackage{graphicx}
\usepackage{hyperref}
\usepackage{bbm}

\usepackage[table]{xcolor}
\definecolor{rankblue}{HTML}{DCEBFA}
\definecolor{rankgreen}{HTML}{DDF3E6}
\definecolor{rankpurple}{HTML}{EDE2F7}

\usepackage[utf8]{inputenc} 
\usepackage[T1]{fontenc}    
\usepackage{url}            
\usepackage{booktabs}       
\usepackage{amsfonts}       
\usepackage{nicefrac}       
\usepackage{microtype}      
\usepackage{xcolor}         
\usepackage{amsthm}
\usepackage{enumitem}
\usepackage{subfigure}
\usepackage{xspace}
\usepackage{wrapfig}
\usepackage{longtable}
\usepackage[table,dvipsnames]{xcolor}
\usepackage{algorithm}
\usepackage{algpseudocode}
\usepackage{amsmath, amssymb}
\usepackage{cleveref}
\usepackage{multirow}
\usepackage{graphicx}
\usepackage{tikz}
\usepackage[framemethod=TikZ]{mdframed}
\usetikzlibrary{arrows.meta,backgrounds,calc,decorations.pathreplacing,fit,positioning}

\setcitestyle{numbers,square}

\newtheorem{theorem}{Theorem}[section]

\theoremstyle{definition}
\newtheorem{definition}[theorem]{Definition}

\theoremstyle{remark}

\mdfdefinestyle{observationbox}{
    linecolor=white,
    outerlinewidth=1pt,
    roundcorner=5pt,
    innertopmargin=4pt,      
    innerbottommargin=3pt,
    innerrightmargin=5pt,
    innerleftmargin=5pt,
    skipabove=6pt,           
    skipbelow=6pt,
    backgroundcolor=blue!6!white
}

\surroundwithmdframed[style=observationbox]{observation}

\newif\ifshowcomments
\showcommentsfalse  

\newcommand{\ourmethod}{\textsc{MuRA}\xspace}

\usepackage{amsmath,amsfonts,bm}

\def\eqref#1{equation~\ref{#1}}

\def\1{\bm{1}}

\def\va{{\bm{a}}}

\def\vc{{\bm{c}}}

\def\ve{{\bm{e}}}

\def\vh{{\bm{h}}}

\def\vk{{\bm{k}}}

\def\vo{{\bm{o}}}
\def\vp{{\bm{p}}}
\def\vq{{\bm{q}}}

\def\vs{{\bm{s}}}

\def\vv{{\bm{v}}}

\def\vx{{\bm{x}}}
\def\vy{{\bm{y}}}
\def\vz{{\bm{z}}}

\def\mC{{\bm{C}}}

\def\mK{{\bm{K}}}

\def\mM{{\bm{M}}}

\def\mQ{{\bm{Q}}}

\def\mS{{\bm{S}}}

\def\mU{{\bm{U}}}
\def\mV{{\bm{V}}}
\def\mW{{\bm{W}}}
\def\mX{{\bm{X}}}
\def\mY{{\bm{Y}}}
\def\mZ{{\bm{Z}}}

\DeclareMathAlphabet{\mathsfit}{\encodingdefault}{\sfdefault}{m}{sl}
\SetMathAlphabet{\mathsfit}{bold}{\encodingdefault}{\sfdefault}{bx}{n}

\def\sN{{\mathbb{N}}}

\def\sR{{\mathbb{R}}}

\newcommand{\softmax}{\mathrm{softmax}}

\newcommand{\te}{\mathsf{E}} 
\newcommand{\rnn}{\mathsf{R}} 
\newcommand{\capacity}{C} 
\newcommand{\attn}{\mathsf{A}} 
\newcommand{\sattn}{\mathsf{S}} 
\newcommand{\hybrid}{\mathsf{B}}
\newcommand{\mlp}{\mathsf{MLP}} 
\newcommand{\vkappa}{\bm{\kappa}} 
\newcommand{\ue}{\mathsf{U}} 
\newcommand{\layer}{\mathsf{L}}

\DeclareMathOperator*{\argmax}{arg\,max}

\title{Improving the Expressive Power of Hybrid Sequence Models via Multi-Hop Retrieval}

\author{
  Yam Eitan$^{1}$ \quad
  Gilad Yehudai$^{1}$ \quad
  Haggai Maron$^{1,2}$ 
  \\
  \\
  $^{1}$Technion -- Israel Institute of Technology \quad
  $^{2}$NVIDIA
}

\date{April 2026}

\begin{document}

\maketitle

\begin{abstract}
Hybrid architectures offer a promising approach to efficient language modeling by combining a small number of full-attention layers with a larger number of efficient sequence layers, such as sparse attention and recurrent layers. While full-attention layers maintain a KV cache that grows linearly with sequence length, efficient layers typically rely on separate, fixed-size states rather than accessing this richer existing memory. Recent work has begun exploring ways to allow efficient layers access to the KV cache, but the theoretical benefits of such access, and the extent to which existing approaches can realize them remain poorly understood. We investigate these questions by analyzing the expressive power of individual hybrid blocks on function composition, a well-studied task in sequence-model expressivity associated with multi-hop reasoning.
We first establish that both conventional hybrid blocks and those with top-$\kappa$ KV-cache access require $\Omega(\sqrt{n})$ capacity to compute even two-function composition over a domain of size $n$.  To overcome this limitation, we introduce the \emph{Multi-hop Retrieval Attention} (\ourmethod) hybrid block, a simple architecture that enables multi-hop retrieval from a shared KV cache while preserving asymptotic computational and memory efficiency. We prove that a single fixed size $m$-hop \ourmethod block can compute $m$-function composition, whereas computing $(m+1)$-function composition with the same retrieval depth requires polynomial model size, establishing a strict expressivity hierarchy in retrieval depth. Synthetic experiments support this hierarchy, showing that increasing retrieval depth enables learning progressively deeper compositions, substantially outperforming existing hybrid architectures. Together, our results identify multi-hop retrieval as a promising mechanism for improving the expressive power of efficient hybrid architectures.

\end{abstract}



\section{Introduction}
\looseness=-1
Transformers have become the dominant architecture for language modeling, but their computational and memory costs pose increasing challenges as models and context lengths grow. Full attention requires quadratic computation in sequence length and, during autoregressive inference, stores previously computed keys and values to avoid recomputation. Specifically, each attention layer maintains a \emph{key--value (KV) cache}  $\mK=[\vk_1,\ldots,\vk_n]$ and $\mV=[\vv_1,\ldots,\vv_n]$, consisting of a key and value vector for each preceding token. The memory required by this cache grows linearly with the context length, becoming a major bottleneck for long-context inference \citep{zhang2023h2o,li2024snapkv,xiao2024efficient,gelberg2026compressibility}. These costs have motivated the development of \emph{efficient sequence layers}, including sparse and local attention mechanisms~\citep{child2019generating,beltagy2020longformer,zaheer2020big} and, more recently, efficient recurrent layers ~\citep{gu2021efficiently,gu2023mamba,yang2024gated}. Recurrent layers process sequences in linear time and, during autoregressive inference, maintain only a fixed-size hidden state $\vh_n$ that summarizes the preceding context. Despite these efficiency gains, models relying exclusively on efficient sequence layers often underperform full-attention transformers on long context tasks~\citep{waleffe2024empirical,ren2025exploring}.

Nevertheless, recurrent layers have found considerable success in \emph{hybrid architectures}. These architectures are typically organized into repeating \emph{hybrid blocks}, each combining a single full-attention layer with several recurrent layers, reducing the frequency of expensive attention operations. This design has been adopted by prominent models such as Kimi K3~\citep{team2026kimi}, Gemma~3~\citep{gemmateam2025gemma3}, Nemotron~3 Ultra~\citep{blakeman2026nemotron}, and OLMo Hybrid~\citep{merrill2026olmohybrid}, demonstrating its effectiveness at scale. However, recurrent layers were originally designed to operate independently of full attention. Consequently, even within a hybrid block, they typically maintain their own fixed-size hidden states $\vh_t$, without direct access to the richer KV cache already maintained by the full-attention layer. This raises a natural question: \emph{can recurrent layers benefit from direct access to this existing memory without sacrificing their computational efficiency?} Recent work has begun exploring this possibility. Some approaches~\citep{brandon2024crosslayer,qian2026dart} share KV caches across layers or  apply attention to a subset of past recurrent states, reducing memory requirements but retaining quadratic attention computation. In contrast, other methods~\citep{gao2026hysparse,hysparse2} reuse a full-attention layer's KV cache for sparse top-$\kappa$ retrieval based on its attention scores, allowing subsequent layers to access shared memory without increasing the asymptotic time or memory complexity of hybrid blocks. While these methods have shown promising empirical results, their designs remain largely empirically motivated, with limited theoretical understanding of how different memory-access mechanisms affect model capabilities.



We aim to provide a principled foundation for designing such layers by investigating whether shared-memory access can improve the expressive power of hybrid blocks, and whether existing retrieval mechanisms fully realize this potential. Expressivity is particularly relevant in this setting: computations that can be realized within a single block of one architecture may require several blocks of another, introducing additional expensive full-attention operations and KV caches. More expressive blocks could therefore enable more efficient hybrid architectures.

\looseness=-1
Our analysis focuses on \emph{function composition}, a well-studied task in sequence-model expressivity~\citep{peng2024limitations,kozachinskiy2025strassen,ye2026hybridhierarchy} that requires successive dependent retrievals, as often arise in multi-hop reasoning. We establish a fundamental expressive limitation shared by both conventional hybrid blocks with independent recurrent states and those whose recurrent layers access a shared KV cache through  top-$\kappa$ retrieval. Specifically, we prove that both require size $\Omega(\sqrt{n})$, measured by the product of width, depth, and precision, to compute even two-function composition over a domain of size $n$.

\looseness=-1
Motivated by this limitation, we introduce \textbf{Mu}lti-\textbf{H}op \textbf{R}etrieval \textbf{A}ttention (\ourmethod), which extends top-$\kappa$ retrieval to provide efficient sequence layers with richer access to the full-attention layer's KV cache. As in top-$\kappa$ hybrid blocks, \ourmethod first identifies each token's highest-scoring causal predecessors using the full-attention scores. However, rather than retrieving only the immediate predecessors of a query token, $(\kappa_1,\ldots,\kappa_m)$-\ourmethod recursively expands the candidate set: at each hop $r$, it retrieves the top-$\kappa_r$ predecessors of the tokens reached at the previous hop, allowing access to tokens up to $m$ hops away. Each subsequent sparse layer computes new queries from the current token representations and attends to these candidates using the original full-attention keys and values. The candidate sets are computed once and remain fixed throughout the block, allowing the sparse layers to refine their queries without recomputing retrieval routes. By enabling retrieved tokens to guide further retrievals, \ourmethod supports complex chains of dependent retrievals while preserving the efficiency of existing hybrid designs: for fixed retrieval parameters, sparse retrieval layers require linear computation in sequence length and maintain constant-size states, introducing no additional KV caches.

Theoretically, we prove that for any fixed $m$, a single $m$-hop \ourmethod block can compute compositions of $m$ functions with constant width and depth and only logarithmic precision. To our knowledge, this provides \emph{the first architecture with a single full-attention layer that is provably capable of such multi-step function composition}. We further establish the importance of retrieval depth, showing that computing compositions of $m+1$ functions with at most $m$ retrieval hops requires size $\Omega(n^{1/(m+1)})$. Together, these results establish a strict hierarchy in expressive power as retrieval depth increases. This separation holds even when the total number of retrieved KV entries is matched, demonstrating that the structure of retrieval, rather than merely the number of retrieved tokens, is critical to expressivity.

Empirically, we evaluate \ourmethod on synthetic function-composition tasks to investigate whether its theoretical advantages can also be realized through training. For $m\in\{1,2,3,4\}$, we find that models with a single $m$-hop \ourmethod block achieve near-perfect test accuracy on compositions of up to $m$ functions, while struggling with deeper compositions. In contrast, the existing hybrid architectures we evaluate struggle with compositions involving more than one function. These results demonstrate that the expressive advantages of \ourmethod can be learned in practice on controlled tasks, supporting multi-hop retrieval as a promising design principle for more expressive hybrid architectures.


In summary, our main contributions are:
\begin{itemize}
    \item \textbf{Expressivity limitations of existing hybrid architectures.}
    We establish an $\Omega(\sqrt{n})$ size lower bound for single-block hybrid architectures on function composition tasks, covering both conventional blocks, and those whose layers access a shared KV cache through top-$\kappa$ retrieval.

    \item \textbf{Multi-Hop Retrieval Attention.}
    We introduce \ourmethod, which enables recurrent layers to perform multi-hop retrieval from a shared KV cache while preserving the computational and memory efficiency of existing hybrid architectures.

    \item \textbf{Expressivity improvements.}
    We prove that a single $m$-hop \ourmethod block can compute compositions of $m$ functions with constant width and depth and logarithmic precision, and establish a strict expressivity hierarchy in retrieval depth, even under matched retrieval budgets.

    \item \textbf{Empirical validation.}
    We demonstrate that this hierarchy emerges through training on synthetic function-composition tasks with up to four retrieval hops, substantially outperforming existing hybrid baselines.
\end{itemize}

\section{Hybrid Blocks: Existing Architectures and Expressivity Limitations}

We begin this section by defining the three main components of hybrid blocks: full-attention layers, recurrent layers, and top-$\kappa$ access recurrent layers. We then introduce a unified framework capturing commonly used hybrid architectures, as well as recent designs that allow recurrent layers to access the full-attention KV cache. Finally, we establish lower bounds on the size of hybrid blocks required to solve \emph{function composition},  a well-studied task in sequence-model expressivity~\citep{peng2024limitations,ye2026hybridhierarchy} involving successive dependent retrievals, as often arise in multi-hop reasoning.

\subsection{Hybrid Block Components and Structure}

We first formalize the three classes of layers used to construct hybrid blocks. Throughout, bold lowercase and uppercase letters denote vectors and matrices, respectively (e.g., $\vv$ and $\mX$), and $[n]=\{1,\dots,n\}$.

\begin{definition}
A \emph{multi-head attention layer} $\attn$ of width $d$ consists of $h$ attention heads, each of dimension $d/h$. Given an input sequence $\mX=[\vx_1,\ldots,\vx_N]^\top\in\sR^{N\times d}$, each head $j\in[h]$ computes queries $\mQ_j=\mX\mW_{Q,j}$, keys $\mK_j=\mX\mW_{K,j}$, and values $\mV_j=\mX\mW_{V,j}$, using learned projections $\mW_{Q,j},\mW_{K,j},\mW_{V,j}\in\sR^{d\times (d/h)}$. The layer computes
\begin{equation}
\begin{aligned}
\mathrm{head}_j(\mX)
&=\softmax\!\left(
\frac{\mQ_j\mK_j^\top}{\sqrt{d/h}}+\mM
\right)\mV_j,\\
\attn(\mX)
&=\mX+[\mathrm{head}_j(\mX)]_{j=1}^{h}\mW_O.
\end{aligned}
\end{equation}
Here, $\mW_O\in\sR^{d\times d}$ is a learned output projection, the softmax is applied row-wise, and the causal mask $\mM\in\sR^{N\times N}$ satisfies $\mM_{i,t}=0$ for $t\leq i$ and $\mM_{i,t}=-\infty$ otherwise. For each head $j$, we define the \emph{attention score} $\alpha_{i,t}^{(j)}$ from token $i$ to token $t$ as the $(i,t)$ entry of the softmax matrix above.
\end{definition}

\looseness=-1
For a fixed layer width, full attention processes a sequence of length $N$ in $O(N^2)$ time, computing pairwise query--key similarities and aggregating the corresponding values. During autoregressive inference, the keys and values of preceding tokens remain unchanged and are therefore stored in a \emph{KV cache}, avoiding their recomputation when processing new tokens. This requires $O(N)$ memory per layer, and reduces the computational complexity for each new token to  $O(N)$. These costs motivate using full-attention layers sparingly and interleaving them with more efficient recurrent layers, which we define next.

\begin{definition}
A \emph{recurrent layer} $\rnn$ of width $d$ is specified by a fixed initial state $\vh_0\in\sR^d$ and two functions $f,g:\sR^d\times\sR^d\to\sR^d$, representing the state transition and output maps, respectively. Given an input sequence $\mX=[\vx_1,\dots,\vx_N]^\top\in\sR^{N\times d}$, the layer computes
\begin{equation}
\vh_i=f(\vh_{i-1},\vx_i),
\qquad
\vy_i=g(\vh_i,\vx_i),
\end{equation}
and outputs $\rnn(\mX)=[\vy_1,\dots,\vy_N]^\top$.
\end{definition}

\looseness=-1
We leave the parameterization of $f$ and $g$  unrestricted to encompass a range of widely used recurrent architectures~\citep{gu2021efficiently,gu2023mamba,yang2024gated}. Our framework also captures efficient attention variants, including sliding-window attention~\citep{beltagy2020longformer} and kernel-based linear attention~\citep{katharopoulos2020transformers}. See appendix \ref{app:preliminaries} for further details.

For fixed layer width, recurrent layers process a sequence of length $N$ in $O(N)$ time.  Additionally, during autoregressive inference, each new token at position $t$ is processed using only the preceding state $\vh_{t-1}$, requiring constant memory and computation per token. While this makes recurrent layers significantly more efficient than full attention, their bounded memory can limit long-context performance~\citep{waleffe2024empirical,ren2025exploring}, motivating hybrid architectures that combine recurrent and full-attention layers.

Since full attention maintains a richer memory in the form of a growing KV cache, recent approaches~\citep{gao2026hysparse,hysparse2} seek to exploit this resource by allowing recurrent layers to access a fixed number of the highest-scoring KV entries per token, aiming to improve performance while preserving efficiency. This motivates the following definition.

\begin{definition}
\label{def:kappa_access_layer}
Fix $\kappa\in\sN$ and a full-attention layer $\attn$ of width $d$ with $h$ heads. For each head $j\in[h]$, let $\mK^j=[\vk_1^j,\dots,\vk_N^j]^\top$ and $\mV^j=[\vv_1^j,\dots,\vv_N^j]^\top$, both in $\sR^{N\times(d/h)}$, denote its cached keys and values. For each position $i\in[N]$ and head $j\in[h]$, let $\mathcal N_j^{(\kappa)}(i)$ consist of position $i$ and the $\min\{\kappa,i-1\}$ preceding positions with the highest attention scores $\alpha_{i,t}^{(j)}$, breaking ties in favor of more recent positions. The corresponding retrieved KV entries are
\begin{equation}
\mC_i =
\Bigl[
\bigl[\vk_t^j,\vv_t^j\bigr]_{t\in\mathcal N_j^{(\kappa)}(i)}
\Bigr]_{j=1}^h.
\end{equation}
A \emph{top-$\kappa$ access recurrent layer} $\rnn_\kappa$ of width $d$ is defined by a fixed initial state $\vh_0\in\sR^d$ and two functions $f$ and $g$, specifying the state transition and output, respectively. Given an input sequence $\mX=[\vx_1,\dots,\vx_N]^\top\in\sR^{N\times d}$, the layer computes
\begin{equation}
\vh_i=f(\vh_{i-1},\vx_i,\mC_i),
\qquad
\vy_i=g(\vh_i,\vx_i,\mC_i),
\end{equation}
and outputs $\rnn_\kappa(\mX)=[\vy_1,\dots,\vy_N]^\top$.
\end{definition}

Compared with standard recurrent layers, top-$\kappa$ access recurrent layers incur additional computation only for computing the neighborhoods $\mathcal N_j^{(\kappa)}(i)$ and retrieving the corresponding KV entries $\mC_i$. As demonstrated by \citet{gao2026hysparse}, attention-based neighborhood selection can be performed without increasing the asymptotic complexity of full attention, and can be implemented efficiently in practice using algorithms such as FlashAttention\cite{dao2022flashattention}. For fixed $\kappa$ and layer width, retrieving $\mC_i$ requires constant time per token. Furthermore, both operations need only be performed once per full-attention layer, and their results can be shared across multiple subsequent recurrent layers. Having defined these three classes of layers, we now formalize the hybrid blocks considered in this work.

\begin{definition}
\label{def:top_k_access_block}
For nonnegative integers $\kappa$, $L_{\mathrm{pre}}$ and $L_{\mathrm{post}}$, a \emph{top-$\kappa$ access hybrid block} $\hybrid_\kappa$ of width $d$ consists of $L_{\mathrm{pre}}$ standard recurrent layers $\rnn^{\mathrm{pre}}_{1}, \dots, \rnn^{\mathrm{pre}}_{L_{\mathrm{pre}}}$, a causal full-attention layer $\attn$, and $L_{\mathrm{post}}$ top-$\kappa$ access recurrent layers $\rnn^{\mathrm{post}}_{1}, \dots, \rnn^{\mathrm{post}}_{L_{\mathrm{post}}}$ . All post-attention layers have access to layer $\attn$' KV cache. For an input sequence $\mX\in\sR^{N\times d}$, the block computes
\begin{equation}
\hybrid_\kappa(\mX)
= \bigl(
\rnn^{\mathrm{post}}_{L_{\mathrm{post}}}
\circ\cdots\circ\rnn^{\mathrm{post}}_1
\circ\attn
\circ\rnn^{\mathrm{pre}}_{L_{\mathrm{pre}}}
\circ\cdots\circ\rnn^{\mathrm{pre}}_1
\bigr)(\mX).
\end{equation}

\end{definition}

Definition~\ref{def:top_k_access_block} captures several prominent hybrid block designs.\footnote{In practice, hybrid models usually use attention layers as part of transformer blocks, typically including layer normalization and position-wise feed-forward networks. We omit these token-wise operations from our theoretical analysis, as they can be incorporated into adjacent recurrent layers without affecting our results, but retain them in our experiments; see Appendix~\ref{app:preliminaries}.} In particular, top-$0$ access blocks represent conventional hybrid architectures, such as those used by Kimi K3~\citep{team2026kimi}, Gemma 3~\citep{gemmateam2025gemma3}, Nemotron 3 Ultra~\citep{blakeman2026nemotron}, and OLMo Hybrid~\citep{merrill2026olmohybrid}, while top-$\kappa$ access blocks capture the more recent shared-KV mechanisms introduced by HySparse~\citep{gao2026hysparse} and HySparse2~\citep{hysparse2}.

Top-$\kappa$ access offers one approach to reusing the existing KV cache within efficient layers. We next investigate whether such shared-memory access can enhance the expressivity of hybrid blocks, and whether top-$\kappa$ retrieval is sufficient to realize these benefits. Studying expressivity at the block level is particularly relevant for designing efficient hybrid models: increasing the capabilities of individual blocks through relatively inexpensive sequence layers could reduce the number of full-attention layers required, thereby saving computation and KV-cache memory.

\subsection{Expressivity Limitations of Hybrid Blocks on Function Composition.}

\looseness=-1
Below, we establish expressivity limitations for individual top-$\kappa$ access hybrid blocks. We treat each block as a sequence model, embedding input tokens into vectors of width $d$ and applying an unembedding layer to the final token representation to produce the output (see Appendix~\ref{app:preliminaries} for details). Throughout, we assume $p$ bits of precision per scalar, where $p$ may scale with the sequence length $N$. Our analysis focuses on \emph{function composition}, a standard task in sequence-model expressivity analysis~\citep{peng2024limitations,kozachinskiy2025strassen,chen2024theoretical,ye2026hybridhierarchy}.

\begin{definition}
\label{def:function_composition}
Given integers $m,n \in \sN$, functions
$f_1,\ldots,f_m:[n]\to[n]$ and a query $x\in[n]$,
we construct a sequence of length $N=mn+1$ over the
alphabet $[n]$ by concatenating the function tables
and appending the query:
\begin{equation}
\begin{aligned}
\va(f_1,\ldots,f_m,x)=\bigl(&f_1(1),\ldots,f_1(n),\ldots,f_m(1),\ldots,f_m(n),x\bigr).
\end{aligned}
\end{equation}
The \emph{$m$-function composition} task requires a
sequence model to output, at the final position,
\begin{equation}
F_{m,n}(\va(f_1,\ldots,f_m,x))
=(f_1\circ\cdots\circ f_m)(x).
\end{equation}
\end{definition}

\looseness=-1
Function composition captures a fundamental aspect of multi-hop reasoning, in which solving a problem requires multiple dependent retrievals. For example, answering a question about the birthplace of a person's father requires first identifying the father and then retrieving his birthplace. To analyze the ability of hybrid blocks to perform such computations, we establish a connection between their expressive power and one-way multiparty number-on-the-forehead (NOF) communication protocols~\citep{viola2009one}, defined below. 

\begin{definition}
\label{def:nof_protocol}
For $m,n\in\sN$, a \emph{one-way $(m+1)$-player number-on-the-forehead (NOF) communication protocol} consists of $m+1$ players $P_1,\ldots,P_{m+1}$ with a fixed communication procedure. Given input functions $f_1,\ldots,f_m:[n]\to[n]$ and a query $x\in[n]$, each player $P_j$, for $j\in[m]$, has access to $x$ and all functions except $f_j$, while player $P_{m+1}$ has access to all functions but not $x$.

Communication proceeds in the fixed order $P_{m+1}\to P_m\to\cdots\to P_1$. Each player except $P_1$ sends a single message, determined by its available inputs and all preceding messages according to the fixed protocol. Finally, $P_1$ produces an output based on its available information and the received messages. The \emph{communication cost} is the maximum, over all possible inputs, of the sum of the bit lengths of all transmitted messages.
\end{definition}

A classical result in communication complexity~\citep{viola2009one} establishes that any one-way $(m+1)$-player NOF protocol computing $m$-function composition for all possible inputs requires $\Omega(n^{1/m})$ bits of communication. We use this lower bound throughout our analysis, beginning with its three-player case ($m=2$) to establish a capacity lower bound for top-$\kappa$ access hybrid blocks.

\begin{theorem}
\label{thm:fixed_access_lower_bound}
Let $\hybrid_\kappa$ be a top-$\kappa$ access hybrid block
with $L_{\mathrm{pre}}$ pre-attention layers and
$L_{\mathrm{post}}$ post-attention layers, width $d$, and precision $p$.
Define its capacity as
\begin{equation}
\capacity(\hybrid_\kappa)
=  pd(L_{\mathrm{post}}+\kappa+1).
\end{equation}
If for some integer $m\geq2$, $\hybrid_\kappa$ computes $m$-function composition for all sequences encoding functions $f_1,\ldots,f_m:[n]\to[n]$
and queries $x\in[n]$, then
\begin{equation}
\capacity(\hybrid_\kappa)=\Omega(\sqrt{n}).
\end{equation}
\end{theorem}

We outline the main argument below, with the full proof provided in Appendix~\ref{app:fixed_access_proof}.

\begin{proof}[Proof sketch]
We show that any top-$\kappa$ access hybrid block $\hybrid_\kappa$ computing $m$-function composition induces a three-player NOF protocol for two-function composition with communication cost $O(\capacity(\hybrid_\kappa))$. For $m>2$, we fix $f_3,\ldots,f_m$ to be the identity. Given inputs $f_1,f_2,x$, the players jointly simulate the block's forward pass on the corresponding sequence of length $N=mn+1$.

\textbf{Player} $P_3$ (knows $f_1,f_2$):
Since all tokens preceding the final query $x$ are known to $P_3$, it can compute their representations and recurrent states throughout the block. It transmits the states of the last pre-attention recurrent layer at positions $n$ and $N-1$, together with the states of all post-attention layers at position $n$. 
\textbf{Player} $P_2$ (knows $f_1,x$):
Using the received pre attention state at position $N-1$,  $P_2$ can reconstruct the last token's attention query, as well as  all keys and values corresponding to the first function table. For each attention head, it computes the contribution of these tokens to the attention output, represented by an unnormalized weighted sum and its softmax normalization factor. It also computes the top-$\kappa$ highest-scoring KV entries within the first $n$ tokens along with their attention scores. These quantities are transmitted to $P_1$.
\textbf{Player} $P_1$ (knows $f_2,x$):
Using the recurrent states received from $P_3$, player $P_1$ reconstructs the representations of tokens $n+1,\dots, N$, and computes their contribution to the final attention output. Combining this with the quantities received from $P_2$, it recovers the full-attention output and the globally selected top-$\kappa$ KV entries. Finally, using the post-attention recurrent states provided by $P_3$, it completes the forward pass and outputs $f_1(f_2(x))$.

The total communication cost of the protocol above is $O(\capacity(\hybrid_\kappa))$. The classical three-player NOF lower bound therefore implies $\capacity(\hybrid_\kappa)=\Omega(\sqrt{n})$.
\end{proof}
\paragraph{Remark.}
\citet{ye2026hybridhierarchy} also establish polynomial lower bounds for hybrid architectures, including those with multiple full-attention layers. However, specializing their result to hybrid blocks using a single full-attention layer yields an $N^{1/1024}$ lower bound on model capacity and assumes a bound on the number of efficient sequence layers. In contrast, we establish an $\Omega(\sqrt{n})$ capacity lower bound that allows arbitrary recurrent depth and also extends to top-$\kappa$ memory access.

\section{Multi-Hop Retrieval from Shared Attention Memory}
\label{sec:method}

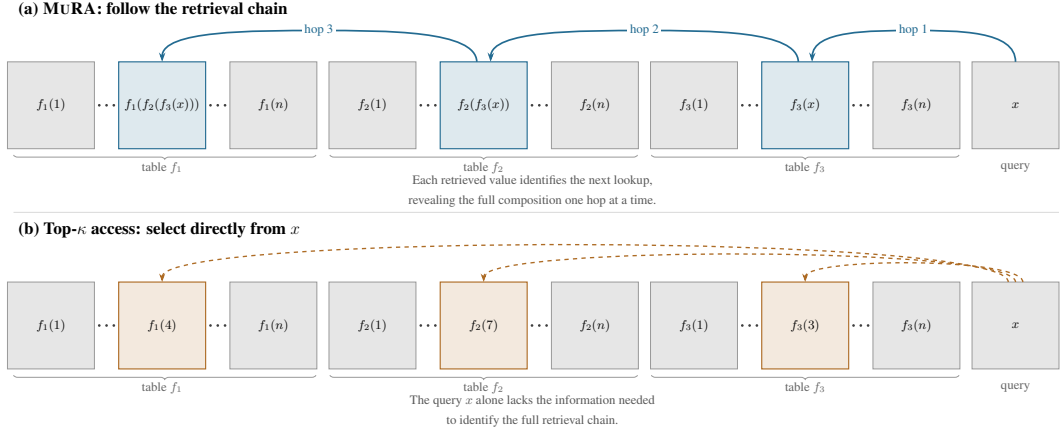
\begin{figure}[t]
\centering
\resizebox{\linewidth}{!}{
\begin{tikzpicture}[
  font=\fontsize{9}{11}\selectfont,
  token/.style={draw=black!40, fill=black!10, minimum size=2.15cm,
    inner sep=0pt, outer sep=0pt, line width=.65pt, align=center},
  retrieved a/.style={draw=muraPath, fill=muraPath!15, line width=.8pt},
  retrieved b/.style={draw=muraDirect, fill=muraDirect!15, line width=.8pt},
  chain/.style={-{Stealth[length=2.3mm,width=1.6mm]}, draw=muraPath, line width=1.2pt},
  direct/.style={-{Stealth[length=2.1mm,width=1.5mm]}, draw=muraDirect,
    line width=.95pt, dashed},
  tablebrace/.style={decorate, decoration={brace,mirror,amplitude=3pt},
    draw=black!40, line width=.5pt},
]
\definecolor{muraPath}{HTML}{236B91}
\definecolor{muraDirect}{HTML}{AD6B24}

\begin{scope}
  \node[anchor=west, font=\large\bfseries] at (-.925,2.36)
    {(a) \ourmethod: follow the retrieval chain};

  \node[token] (a11) at (0,0) {$f_1(1)$};
  \node[token, retrieved a] (a12) at (2.75,0) {$f_1(f_2(f_3(x)))$};
  \node[token] (a13) at (5.5,0) {$f_1(n)$};
  \node[token] (a21) at (7.95,0) {$f_2(1)$};
  \node[token, retrieved a] (a22) at (10.7,0) {$f_2(f_3(x))$};
  \node[token] (a23) at (13.45,0) {$f_2(n)$};
  \node[token] (a31) at (15.9,0) {$f_3(1)$};
  \node[token, retrieved a] (a32) at (18.65,0) {$f_3(x)$};
  \node[token] (a33) at (21.4,0) {$f_3(n)$};
  \node[token] (ax) at (23.85,0) {$x$};
  \foreach \x in {1.375,4.125,9.325,12.075,17.275,20.025}
    \foreach \offset in {-.18,0,.18}
      \fill[black!75] (\x+\offset,0) circle[radius=.033cm];

  \draw[chain] (ax.north)
    .. controls (23.85,1.85) and (22.426,1.85) .. (21.327,1.85)
    .. controls (20.244,1.85) and (18.841,1.85) .. ([xshift=1.7mm]a32.north);
  \node[fill=white, inner sep=1.5pt, text=muraPath] at (21.327,1.85) {hop 1};
  \draw[chain] ([xshift=-1.7mm]a32.north)
    .. controls (18.459,1.85) and (16.315,1.85) .. (14.675,1.85)
    .. controls (13.035,1.85) and (10.891,1.85) .. ([xshift=1.7mm]a22.north);
  \node[fill=white, inner sep=1.5pt, text=muraPath] at (14.675,1.85) {hop 2};
  \draw[chain] ([xshift=-1.7mm]a22.north)
    .. controls (10.509,1.85) and (8.309,1.85) .. (6.628,1.85)
    .. controls (4.95,1.85) and (2.75,1.85) .. (a12.north);
  \node[fill=white, inner sep=1.5pt, text=muraPath] at (6.628,1.85) {hop 3};

  \draw[tablebrace] (-.925,-1.23) -- (6.554,-1.23);
  \draw[tablebrace] (6.896,-1.23) -- (14.504,-1.23);
  \draw[tablebrace] (14.846,-1.23) -- (22.454,-1.23);
  \node[text=black!60] at (2.75,-1.54) {table $f_1$};
  \node[text=black!60] at (10.7,-1.54) {table $f_2$};
  \node[text=black!60] at (18.65,-1.54) {table $f_3$};
  \node[text=black!60] at (23.85,-1.54) {query};
  \node[align=center, text=black!65] at (11.907,-2.12)
    {Each retrieved value identifies the next lookup,\\[3pt]
     revealing the full composition one hop at a time.};
\end{scope}

\draw[black!15] (-.925,-2.65) -- (24.925,-2.65);

\begin{scope}[yshift=-5.45cm]
  \node[anchor=west, font=\large\bfseries] at (-.925,2.36)
    {(b) Top-$\kappa$ access: select directly from $x$};

  \node[token] (b11) at (0,0) {$f_1(1)$};
  \node[token, retrieved b] (b12) at (2.75,0) {$f_1(4)$};
  \node[token] (b13) at (5.5,0) {$f_1(n)$};
  \node[token] (b21) at (7.95,0) {$f_2(1)$};
  \node[token, retrieved b] (b22) at (10.7,0) {$f_2(7)$};
  \node[token] (b23) at (13.45,0) {$f_2(n)$};
  \node[token] (b31) at (15.9,0) {$f_3(1)$};
  \node[token, retrieved b] (b32) at (18.65,0) {$f_3(3)$};
  \node[token] (b33) at (21.4,0) {$f_3(n)$};
  \node[token] (bx) at (23.85,0) {$x$};
  \foreach \x in {1.375,4.125,9.325,12.075,17.275,20.025}
    \foreach \offset in {-.18,0,.18}
      \fill[black!75] (\x+\offset,0) circle[radius=.033cm];

  \draw[direct] ([xshift=-1.8mm]bx.north)
    .. controls (23.67,2.29) and (2.75,2.29) .. (b12.north);
  \draw[direct] (bx.north)
    .. controls (23.85,1.98) and (10.7,1.98) .. (b22.north);
  \draw[direct] ([xshift=1.8mm]bx.north)
    .. controls (24.03,1.69) and (18.65,1.69) .. (b32.north);

  \draw[tablebrace] (-.925,-1.23) -- (6.554,-1.23);
  \draw[tablebrace] (6.896,-1.23) -- (14.504,-1.23);
  \draw[tablebrace] (14.846,-1.23) -- (22.454,-1.23);
  \node[text=black!60] at (2.75,-1.54) {table $f_1$};
  \node[text=black!60] at (10.7,-1.54) {table $f_2$};
  \node[text=black!60] at (18.65,-1.54) {table $f_3$};
  \node[text=black!60] at (23.85,-1.54) {query};
  \node[align=center, text=black!65] at (11.907,-2.12)
    {The query $x$ alone lacks the information needed\\[3pt]
     to identify the full retrieval chain.};
\end{scope}
\end{tikzpicture}}
\caption{\textbf{\ourmethod versus top-$\kappa$ retrieval.} Both panels illustrate the same input for $3$-function composition. (a) The construction of Theorem~\ref{thm:mura_main_upper} performs three successive lookups to compute $f_1(f_2(f_3(x)))$. (b) An illustrative attention-score-based top-$\kappa$ selection retrieves several tokens directly from $x$, but fails to reach the target token through the required retrieval chain.}
\label{fig:multi_hop_retrieval}
\end{figure}
\vspace{-3mm}
Theorem~\ref{thm:fixed_access_lower_bound} reveals a fundamental limitation of top-$\kappa$ access hybrid blocks in performing multi-hop reasoning. To address this limitation, we introduce \emph{Multi-hop Retrieval Attention} (\ourmethod), a new hybrid block architecture that allows sparse-attention layers to access the existing KV cache through multiple retrieval hops. We first formally define $(\kappa_1,\ldots,\kappa_m)$-\ourmethod blocks and then establish their expressive advantage over existing hybrid architectures on function composition. In particular, we show that retrieval depth induces a strict expressivity hierarchy: an $m$-hop block can compute $m$-function composition with constant width and logarithmic capacity, whereas computing $(m+1)$-function composition requires polynomial capacity. Importantly, these expressive gains come without increasing the asymptotic computational or memory complexity of existing hybrid blocks.


\subsection{\ourmethod Hybrid Blocks}

\looseness=-1
To overcome the multi-hop expressivity limitations of top-$\kappa$ access hybrid blocks,
we propose a simple extension to the set of KV entries available
to subsequent layers. Rather than retrieving only the tokens
with the highest attention scores relative to the current position,
we also retrieve tokens that are highly attended to by those
already selected, recursively expanding the retrieved set
over multiple hops.

For a fixed full-attention layer $\attn$, recall that
$\mathcal N_j^{(\kappa)}(i)$ denotes the top-$\kappa$ neighborhood
of position $i$ for attention head $j$
(Definition~\ref{def:kappa_access_layer}).
Given retrieval widths $\kappa_1,\ldots,\kappa_m \in \sN$,
we recursively define  multi-hop neighborhood $\mathcal N_j^{(\kappa_1,\ldots,\kappa_r)}(i)$ by
\begin{equation}
\mathcal N_j^{(\kappa_1,\ldots,\kappa_r)}(i)
=\bigcup_{t\in\mathcal N_j^{(\kappa_1,\ldots,\kappa_{r-1})}(i)}
\mathcal N_j^{(\kappa_r)}(t),
\qquad 2\leq r\leq m.
\label{eq:mura_main_retrieval}
\end{equation}
The resulting neighborhoods contain chains of tokens connected
by high-scoring attention links,\footnote{
Since the one-hop neighborhoods $\mathcal N_j^{(\kappa)}(i)$
include position $i$ itself, $m$-hop neighborhoods
contains all positions reachable within at most $m$ hops.
In practice, we also consider a variant that 
restricts retrieval to positions reachable in exactly $m$ hops. We find this variant is sometimes preferable empirically,
while retaining the original convention for our theoretical analysis.
}
potentially exposing information that cannot be accessed by top-$\kappa$ retrieval.
We now define \ourmethod  sparse-attention layers, which attend only to retrieved multi-hop positions.

\begin{definition}
\label{def:mura_attention_layer}
Fix $\vkappa =(\kappa_1,\dots,\kappa_m) \in \sN^m$ and full-attention layer $\attn$ of width $d$ with $h$ heads. For each head $j\in[h]$, denote its cached keys and values by $\mK^j=[\vk_1^j,\dots,\vk_N^j]^\top$ and $\mV^j=[\vv_1^j,\dots,\vv_N^j]^\top$, both in $\sR^{N\times (d/h)}$. A $\vkappa$-\emph{\ourmethod multi-head attention layer} $\sattn$ is specified by $h$ query projection heads $\mW_{Q,j} \in \sR^{d\times(d/h)}$ and $h$ integer tuples $\vkappa_j = (\kappa_1^j,\dots,\kappa^1_{m_j})$ with $m_j\leq m$ and $\kappa^j_i\leq\kappa_i \quad \forall i \in [m_j]$.  Given an input sequence $\mX=[\vx_1,\ldots,\vx_N]^\top\in\sR^{N\times d}$, each head $j\in[h]$ computes queries $\mQ_j=\mX\mW_{Q,j}$ and reuses the full attention keys and values, to compute 
 \begin{equation}
\begin{aligned}
\mathrm{head}_j(\mX)
&=\softmax\!\left(
\frac{\mQ_j\mK_j^\top}{\sqrt{d/h}}+\mM_{\vkappa_j}
\right)\mV_j,\\
\attn(\mX)
&=\mX+[\mathrm{head}_j(\mX)]_{j=1}^{h}\mW_O.
\end{aligned}
\end{equation}
Here, $\mW_O\in\sR^{d\times d}$ is a learned output projection, the softmax is applied row-wise, and the  mask $\mM_{\vkappa_j}\in\sR^{N\times N}$ satisfies $\mM_{i,t}=0$ for $t\in \mathcal N_j^{\vkappa_j}(i)$ and $\mM_{i,t}=-\infty$ otherwise. We refer to $m$ as the retrieval depth of $\sattn$.
\end{definition}

Importantly, \ourmethod preserves the asymptotic efficiency of top-$\kappa$ access layers. For fixed retrieval depths and widths $\vkappa$, each neighborhood $\mathcal N_j^{\vkappa_j}(i)$ contains a constant number of positions independent of the sequence length $N$. Consequently, each sparse-attention layer attends to only a fixed number of KV entries per token, retaining the same asymptotic computational and memory complexity as top-$\kappa$ access. Although the neighborhood size can grow exponentially with retrieval depth $m$, we show later that even the minimal choice $\kappa_1=\cdots=\kappa_m=1$ is sufficient to improve expressive power, and that higher retrieval depth increases expressivity even for a matched number of retrieved tokens.

Beyond these theoretical guarantees, \ourmethod also admits a practical implementation strategy. HySparse~\citep{gao2026hysparse} demonstrates that attention-based routing information can be extracted alongside full attention using a modified FlashAttention kernel, without materializing the full attention matrix, and employs this approach to train 7B language models. Building on this design, multi-hop retrieval requires only an additional traversal of the selected token indices to construct the expanded neighborhoods, which can then be shared across all subsequent sparse-attention layers. The underlying KV cache and full-attention computation remain unchanged, making \ourmethod a natural extension of existing efficient implementations. We now formally define \ourmethod hybrid blocks.

\begin{definition}
\label{def:mura_hybrid_block}
For nonnegative integers  $L_{\mathrm{pre}}$ and $L_{\mathrm{post}}$, and $\vkappa \in \sN^m$ a \emph{$\vkappa$-\ourmethod hybrid block} $\hybrid_{\vkappa}$ of width $d$ consists of $L_{\mathrm{pre}}$ standard recurrent layers $\rnn^{\mathrm{pre}}_{1}, \dots, \rnn^{\mathrm{pre}}_{L_{\mathrm{pre}}}$, a causal full-attention layer $\attn$, and $L_{\mathrm{post}}$ $\vkappa$-\ourmethod attention layers $\sattn^{\mathrm{post}}_{1}, \dots, \sattn^{\mathrm{post}}_{L_{\mathrm{post}}}$. All post-attention layers have access to the KV cache of layer $\attn$. For an input sequence $\mX\in\sR^{N\times d}$, the block computes
\begin{equation}
\hybrid_{\vkappa}(\mX)
= \bigl(
\sattn^{\mathrm{post}}_{L_{\mathrm{post}}}
\circ\cdots\circ\sattn^{\mathrm{post}}_1
\circ\attn
\circ\rnn^{\mathrm{pre}}_{L_{\mathrm{pre}}}
\circ\cdots\circ\rnn^{\mathrm{pre}}_1
\bigr)(\mX).
\end{equation}
\end{definition}

\subsection{The Expressive Power of \ourmethod Hybrid Blocks.}

We study the effects of retrieval depth on the  expressive
power of \ourmethod hybrid blocks. Our first result shows
that even the minimal choice of retrieving a single
predecessor per hop is sufficient to efficiently compute
function composition, with the number of hops matching
the number of composed functions. The full proof appears in Appendix~\ref{app:mura_expressivity}, and illustrated in Figure~\ref{fig:multi_hop_retrieval}.

\begin{theorem}
\label{thm:mura_main_upper}
For every fixed $m\geq1$, a single $(\overbrace{1,\ldots,1}^m)$-\ourmethod block can compute $m$-function composition for every $n\geq2$ using a single attention head, a single \ourmethod sparse attention layer, no recurrent layers, width $d=O(m)$ independent of the sequence length $N$, and precision $p=O(log(N))$. 
\end{theorem}

\begin{proof}[Proof sketch]
Consider an input sequence $\va(f_1,\ldots,f_m,x)$
as in Definition~\ref{def:function_composition}.
We embed each token corresponding
to $f_i(j)$ via a function of its table index $i\in[m]$,
input argument $j\in[n]$, and value $f_i(j)\in[n]$.\footnote{
While this encoding differs from standard positional embeddings, it
can be recovered from them using a single pre-attention
recurrent layer. See Appendix~\ref{app:preliminaries}
for details.
}
We can then construct a full-attention layer $\attn$ such
that, for every $i>1$, the unique top-$1$ predecessor
of the token representing $f_i(j)$ is the token representing
$f_{i-1}(f_i(j))$. Similarly, the top-$1$ predecessor
of the final query token $x$ is the token representing
$f_m(x)$. Consequently, the $m$-hop neighborhood
$\mathcal N^{(1,\ldots,1)}(x)$ contains precisely the
tokens along the composition chain
\begin{equation}
x \to f_m(x) \to (f_{m-1}\circ f_m)(x)
\to \cdots \to (f_1\circ\cdots\circ f_m)(x).
\end{equation}

We set the full-attention output projection $\mW_O$ to zero,
preserving the input representations through the residual connection. A single sparse-attention layer then attends
to the retrieved composition chain. 
Since only the final token in the retrieved chain belongs to the first function table, we choose the full-attention value projection so that the last coordinate of each cached value contains the function value for tokens from the first table and zero otherwise. Setting the sparse-attention query to zero induces uniform attention over the $m+1$ retrieved tokens, yielding $\frac{1}{m+1}(f_1\circ\cdots\circ f_m)(x)$ in the last output coordinate. Rescaling by $m+1$ recovers the desired composition, completing the construction. 
\end{proof}


We next establish a strict expressivity hierarchy in retrieval depth: each additional hop enables higher-order function composition with logarithmic capacity, while blocks with fewer hops require polynomial capacity, regardless of the number of (fixed) retrieved KV entries. This highlights the importance of retrieval depth over retrieval size. The full proof appears in Appendix~\ref{app:mura_lower_proof}.

\begin{theorem}
\label{thm:mura_main_lower}
For $\vkappa \in \sN^m$, with $m\geq2$, let $\hybrid_{\vkappa}$ be a $\vkappa$-\ourmethod model with $\vkappa=(\kappa_1,\dots, \kappa_m)$, $L_{\mathrm{pre}}$ pre-attention layers and
$L_{\mathrm{post}}$ post-attention layers, width $d$, and precision $p$.
Define its capacity as
\begin{equation}
\capacity(\hybrid_{\kappa})
=  pd m\cdot\Pi_{i=1}^m\kappa_i.
\end{equation}
If $\hybrid_{\vkappa}$ computes $(m+1)$-function composition on every
instance of size $n$, then
\[
\capacity(\hybrid_{\vkappa})=\Omega(n^{1/(m+1)}).
\]
\end{theorem}

\begin{proof}[Proof sketch]
We extend the reduction of
Theorem~\ref{thm:fixed_access_lower_bound} to an
$(m+2)$-player NOF protocol for $(m+1)$-function
composition. Given a $\vkappa$-\ourmethod block, the
players simulate its forward pass on sequences $\va(f_1,\ldots,f_{m+1},x)$ of
length $N=(m+1)n+1$.
Player $P_{m+2}$ knows all function tables but not
the query $x$, while each $P_i$, for $i\in[m+1]$,
knows $x$ and all function tables except $f_i$. 

\textbf{Full attention computation:}
Since $P_{m+2}$ knows all tokens preceding $x$,
it can compute their representations and transmit
the recurrent states needed at the boundaries
between function tables.
As in the previous proof, $P_{m+1}$ computes the
contribution of the first $m$ function tables to
the final token's attention output, represented by
an unnormalized weighted sum and its softmax
normalization factor.
Player $P_m$ then adds the contribution of the
remaining tokens corresponding to function $m+1$, recovering the full-attention
output. 

\textbf{Retrieval:} The players simulate the $m$ retrieval hops sequentially, starting with $P_{m+1}$.
At each stage, player $P_i$ computes attention scores with respect  to the previously retrieved tokens  and selects the top-$\kappa_{m+2-i}$
candidates among all tokens, except those corresponding to $f_i$,
the only function table it cannot access.
It transmits these candidates and their attention scores to $P_{i-1}$,
which knows $f_i$ and can therefore complete the global
top-$\kappa_{m+2-i}$ selection.
Using the transmitted recurrent states, $P_{i-1}$
continues the computation and prepares the candidates for the next hop.
In this way, each successive player completes one additional retrieval hop.
Since the block performs only $m$ hops, player $P_1$ can recover
the final retrieved KV entries, complete the forward pass,
and output $(f_1\circ\cdots\circ f_{m+1})(x)$.

The resulting protocol has communication cost $O(\capacity(\hybrid))$.
Applying the $(m+2)$-player NOF pointer-jumping lower bound yields
$\capacity(\hybrid)=\Omega(n^{1/(m+1)})$, as claimed.

\end{proof}

\section{Empirical Evaluation}
\label{sec:experiments}

We empirically validate our theoretical findings on a synthetic function-composition benchmark. Our results show that existing hybrid architectures struggle to learn compositions, whereas \ourmethod successfully learns increasingly complex compositions as retrieval depth grows, demonstrating the benefits of multi-hop access to the shared KV cache. We provide the main experimental details and results below, with a more comprehensive description in Appendix~\ref{app:experimental_details}.


\paragraph{Synthetic benchmark.}
We construct a synthetic dataset of 250,000 sequences, each encoding five functions $f_1,\ldots,f_5:[n]\to[n]$, where $n\in\{25,\ldots,30\}$ and all function values are sampled independently and uniformly from $[n]$. The function tables are serialized as in Definition~\ref{def:function_composition}, omitting the final query token. We train hybrid blocks on a causal token-prediction task that jointly evaluates function composition at depths $m\in[4]$. Specifically, each token representing $f_i(j)$ is assigned multiple composition targets by treating its value as the query $x=f_i(j)$ and applying the preceding functions:
\begin{equation}
f_{i-1}(f_i(j)),\;
(f_{i-2}\circ f_{i-1})(f_i(j)),\;\ldots,\;
(f_1\circ\cdots\circ f_{i-1})(f_i(j)).
\end{equation}
These targets are predicted from each token's final representation.
We evaluate accuracy at each composition depth $m \in [4]$ over all tokens for which the corresponding composition is defined. Models are trained for 500 epochs using cross-entropy loss over all composition targets. The dataset is split into 200,000 training, 25,000 validation, and 25,000 test examples, and we report test accuracy at the checkpoint selected by the best validation score.

\paragraph{Baselines and architecture details.}
All experiments use a single hybrid block comprising of one full-attention layer followed by four efficient sequence layers, each with four heads and hidden dimension $d=16$. All models use the same positional encoding, supervision, and optimizer. We evaluate $\vkappa$-\ourmethod with $\vkappa\in\{(32),(16,1),(11,1,1),(8,1,1,1)\}$, corresponding to retrieval depths of one through four hops. These configurations are chosen to have approximately the same retrieval budget of 32 tokens (33 for $(11,1,1)$). The single-hop configuration $(32)$ corresponds to a simplified version of HySparse~\citep{gao2026hysparse}. To facilitate learning at different composition depths, each head $j$ of a \ourmethod sparse-attention layer uses a retrieval tuple $\vkappa_j=(\kappa_1,\ldots,\kappa_{\min(j,m)})$, allowing earlier heads to access only shallower retrieval neighborhoods. We compare \ourmethod against Gated DeltaNet \cite{yang2024gated}, linear attention~\citep{katharopoulos2020transformers}, and sliding-window attention~\citep{beltagy2020longformer}. For sliding-window attention, we evaluate window sizes of both 32 and 16 tokens. While a window of 32 matches the retrieval budget of \ourmethod, each sliding-window layer maintains its own KV cache, resulting in a combined cache size comparable to an additional full-attention layer in our setting. We therefore also consider a window size of 16 to provide a more memory-efficient baseline.


\begin{table*}[t]
\centering
\caption{Function-composition accuracy (\%), reported as
mean $\pm$ standard deviation over three seeds.
\protect\colorbox{rankblue}{\strut Blue},
\protect\colorbox{rankgreen}{\strut green}, and
\protect\colorbox{rankpurple}{\strut purple}
indicate the first, second, and third highest
mean accuracies in each column, respectively.}
\label{tab:composition_accuracy}
\small
\begin{tabular}{lcccc}
\toprule
Model & $m=1$ & $m=2$ & $m=3$ & $m=4$ \\
\midrule
\ourmethod{} $(32,)$
& $83.05\pm29.35$
& $44.47\pm23.64$
& $5.81\pm3.49$
& $3.68\pm0.11$ \\

\ourmethod{} $(16,1)$
& \cellcolor{rankblue}$100.00\pm0.01$
& \cellcolor{rankgreen}$99.30\pm1.22$
& $30.31\pm23.22$
& $3.73\pm0.10$ \\

\ourmethod{} $(11,1,1)$
& \cellcolor{rankblue}$100.00\pm0.00$
& \cellcolor{rankblue}$100.00\pm0.00$
& \cellcolor{rankgreen}$99.99\pm0.01$
& \cellcolor{rankgreen}$27.81\pm24.98$ \\

\ourmethod{} $(8,1,1,1)$
& \cellcolor{rankblue}$100.00\pm0.00$
& \cellcolor{rankblue}$100.00\pm0.00$
& \cellcolor{rankblue}$100.00\pm0.00$
& \cellcolor{rankblue}$100.00\pm0.00$ \\
\midrule
Linear attention
& \cellcolor{rankpurple}$99.97\pm0.01$
& $57.83\pm10.75$
& $14.97\pm5.66$
& \cellcolor{rankpurple}$8.78\pm0.62$ \\

Sliding window (16)
& \cellcolor{rankblue}$100.00\pm0.00$
& $46.83\pm8.49$
& $3.76\pm0.23$
& $3.72\pm0.08$ \\

Sliding window (32)
& \cellcolor{rankgreen}$99.99\pm0.01$
& \cellcolor{rankpurple}$85.18\pm4.30$
& \cellcolor{rankpurple}$59.86\pm4.18$
& $3.65\pm0.02$ \\

Gated DeltaNet
& $82.07\pm31.03$
& $25.55\pm16.14$
& $13.89\pm10.93$
& $7.97\pm2.48$ \\
\bottomrule
\end{tabular}
\end{table*}


\paragraph{Discussion.}
Table~\ref{tab:composition_accuracy} demonstrates a clear relationship between retrieval depth and the ability to learn function composition. \ourmethod blocks achieve near-perfect accuracy on $m$-function composition when their retrieval depth is at least $m$, while shallower configurations fail to fully learn the task. Interestingly, shallower configurations sometimes achieve accuracy substantially above random guessing, but remain far from perfect and exhibit considerably higher variance across seeds. Consistent with our theoretical findings, the baseline hybrid blocks generally perform well on single-function composition but fail on deeper compositions. Even sliding-window attention with a window size of 32, which requires additional KV caches comparable in total size to another full-attention layer, struggles on more complex function composition.  Together, these results demonstrate the expressive advantage of \ourmethod in learning complex multi-hop computations that remain challenging for existing hybrid architectures.

\section{Related Work}
\label{sec:related_work}

\paragraph{Efficient sequence models and shared attention memory.}
Reducing the computational cost of full attention has motivated extensive research into efficient sequence architectures. Early work explored sparse attention~\citep{child2019generating,beltagy2020longformer,zaheer2020big}, while subsequent approaches developed efficient recurrent layers ~\citep{gu2021efficiently,gu2023mamba,yang2024gated}. These architectures were later incorporated into  hybrid architectures that combine them with full attention, including Kimi K3~\citep{team2026kimi}, Gemma 3~\citep{gemmateam2025gemma3}, Jamba~\citep{lieber2024jamba}, and OLMo Hybrid~\citep{merrill2026olmohybrid}. These architectures typically maintain separate memory for their efficient layers, however recent work explores the idea  of memory sharing: Cross-Layer Attention~\citep{brandon2024crosslayer} shares KV activations across layers, while DART~\citep{qian2026dart} attends to compressed recurrent chunk states. Both reduce memory requirements but still attend over a growing number of entries. HySparse and HySparse2~\citep{gao2026hysparse,hysparse2} instead reuse full-attention KV caches for bounded sparse retrieval. Closely related to our work, Graph Machine~\citep{hou2026graphmachine} also employs multi-hop sparse retrieval, but maintains learned weighted edges and updates KV representations across sparse layers, limiting its memory savings. In contrast, \ourmethod performs retrieval over a fixed attention graph, reusing the original KV cache across all sparse layers.

\paragraph{Expressivity of sequence models.}
The expressive power of sequence models has been extensively studied, with circuit-complexity analyses establishing limitations of log-precision Transformers~\citep{merrill2023parallel} and showing how recurrent depth, padding, and chain of though improve these limitations ~\citep{dehghani2018universal,merrill2026exact,merrill2024expressive}. Other work investigates the architectural resources required for tasks such as sparse averaging, Match3, and function composition~\citep{sanford2023representational,sanford2024parallel,peng2024limitations,chen2024theoretical,kozachinskiy2025strassen}. Related analyses study  limitations of linear RNNs~\citep{merrill2024illusion,sarrof2024expressive}, with architectural modifications such as negative transition eigenvalues improving their state-tracking capabilities~\citep{grazzi2025unlocking}. More recently, \citet{merrill2026olmohybrid} established expressive advantages of combining attention and recurrence, connecting these theoretical findings to the design of a large-scale hybrid model. Closely related to our work, \citet{ye2026hybridhierarchy} established lower bounds for hybrid architectures on sequential function composition, including architectures with multiple full-attention layers. While our analysis focuses on a single full-attention block, we establish a stronger capacity lower bound that allows arbitrary recurrent depth and extends to fixed top-$\kappa$ shared-memory access.


\section{Conclusion}
\label{sec:conclusion}


We studied how sharing attention memory can enhance the expressive power of hybrid blocks without increasing their asymptotic computational or memory complexity. Using function composition as a model of multi-hop reasoning, we first established capacity lower bounds for existing hybrid blocks with either no KV-cache access or fixed top-$\kappa$ retrieval, revealing fundamental limitations of these architectures. To overcome these limitations, we introduced \ourmethod, a simple yet powerful extension that enables multi-hop retrieval from a shared KV cache. Our theoretical results establish a strict expressivity hierarchy in retrieval depth: $m$ hops suffice to compute $m$-function composition with constant width and logarithmic capacity, whereas computing $(m+1)$-function composition with only $m$ hops requires polynomial capacity. Synthetic experiments with comparable retrieval budgets support this hierarchy and demonstrate the expressive advantages of \ourmethod over existing hybrid architectures. Together, our findings highlight retrieval depth as a fundamental architectural principle for designing more expressive hybrid language models.

\section{Acknowledgements}
 HM is supported by the Israel Science Foundation through a personal grant (ISF 264/23) and an equipment grant (ISF 532/23).

\bibliographystyle{plainnat}
\bibliography{references}

\appendix

\section{Additional Preliminaries}
\label{app:preliminaries}

\paragraph{Notation.}
We retain the notation of the main text and introduce new conventions, which will be helpful for the proofs below. 
The vectors $\ve_1,\dots, \ve_m$ represent the standard basis, with dimension specified where it is used, and $\mathbf0_b$ is the zero vector in $\sR^b$. The indicator of a set $E$ is $\mathbbm{1}_E$.
We write $d_h=d/h$ for the dimension of an attention head. 
In the communication proofs, $\vz_t$ denotes the representation at position $t$ immediately before full attention, and $\vq_t^j,\vk_t^j,\vv_t^j$ are its original query, key, and value in head $j$. 
These original vectors, and the neighborhoods determined by the full-attention scores $\alpha_{i,t}^{(j)}$, remain fixed throughout subsequent shared-cache computations. The notation $\operatorname{Top}_{j,k}(u;T)$ denotes the highest scoring $\min\{k,|T|\}$ positions in $T\subseteq[u-1]$, ranked by decreasing score and then decreasing position with respect to query $q_u^j$. 

\paragraph{Sequence modeling: embedding and decoding.}
Throughout this paper, we consider hybrid blocks which take as input sequences $\va=(a_1,\ldots,a_N)$  where each sequence element takes value in a finite alphabet  $a_i \in [n] \quad \forall i \in [N]$. The hybrid block then predicts an output token $\hat a \in [n]$. This is done through two additional components to the hybrid block, the embedding and unembedding layers.  An embedding layer $\te_n:[n]\times \sN \to\sR^d$ produces
\begin{equation}
\te_n(\va)=[\te_n(a_1,1),\ldots,\te_n(a_N,N)]^\top.
\end{equation}
At position $t$, this embedding depends only on $a_t$, and  $t$. Applying a block $\hybrid$ gives $\mY=\hybrid(\te_n(\va))=[\vy_1,\ldots,\vy_N]^\top$. An unembedding layer  $\ue_n:\sR^d\to\sR^n$ then produces output logits. The predicted symbol is
\begin{equation}
\widehat a=\argmax_{a\in[n]}[\mlp_n(\vy_N)]_a.
\label{eq:appendix_decoding}
\end{equation}

In practice, the embedding layer typically combines learned token embeddings with positional encodings:
\begin{equation}
    \te_n(a_i,i)=\vv_{a_i}+\vp_i,
\end{equation}
where $\vv_{a_i}\in\sR^d$ is the learned embedding of token $a_i$ and $\vp_i\in\sR^d$ is a positional encoding, such as a sinusoidal encoding. Similarly, the unembedding layer is typically a linear projection. In this manuscript, we leave both maps unrestricted, as our lower bounds hold in this more general setting. For our upper-bound construction (Theorem~\ref{thm:mura_main_upper}), the required embedding maps can be implemented using standard embeddings followed by a single additional pre-attention recurrent layer, while retaining a linear unembedding.

\paragraph{Normalization and feed-forward networks.}
In practice, sequence layers are typically augmented with layer normalization and position-wise feed-forward networks (FFNs), which we omit from our theoretical formulations for simplicity. Specifically, given an input sequence $\mX=[\vx_1,\dots,\vx_N]^\top\in\sR^{N\times d}$ and a sequence layer $\layer$, we apply the following pre-normalization updates:
\begin{equation}
\begin{aligned}
\mZ &= \operatorname{LN}_1(\mX),\\
\widetilde{\mX} &= \mX+\layer(\mZ),\\
\mY &= \widetilde{\mX}+\operatorname{FFN}(\operatorname{LN}_2(\widetilde{\mX})).
\end{aligned}
\label{eq:appendix_transformer_wrapper}
\end{equation}
We use this structure throughout our experiments. In our theoretical analysis, we omit these components, as they operate independently at each token and can be incorporated into the unrestricted computations of the recurrent layers used by all hybrid models, without affecting the communication-based lower bounds.

\paragraph{Finite precision and attention arithmetic.}
Throughout our analysis, hybrid blocks operate at finite precision $p=p(N)$, specifying the number of bits available to represent each scalar. Two common choices are constant precision, reflecting the fixed numerical precision used in practice, and logarithmic precision $p=\Theta(\log N)$, which is sufficient to distinguish token positions while limiting the information encoded in each scalar.

Since the order of arithmetic operations can affect finite-precision computations, we explicitly specify how attention is evaluated. For each query position $u$ and head $j$, we first compute the query--key dot products by summing coordinate-wise products in a fixed order, obtaining the quantized pre-softmax logits $\gamma_{u,t}^{(j)}$. We then apply a fixed positive, quantized exponential $\operatorname{Exp}_p$ to obtain the unnormalized attention weights
\begin{equation}
w_{u,t}^{(j)}=\operatorname{Exp}_p(\gamma_{u,t}^{(j)}).
\end{equation}
For any set of positions $T\subseteq[u]$, we define the partial unnormalized attention sum and normalization factor as
\begin{equation}
\begin{aligned}
\vs_j(u;T)&=\sum_{t\in T}\operatorname{Quant}_p(w_{u,t}^{(j)}\vv_t^j),\\
D_j(u;T)&=\sum_{t\in T}w_{u,t}^{(j)}.
\end{aligned}
\label{eq:appendix_attention_statistics}
\end{equation}
Here, $\operatorname{Quant}_p$ denotes quantization to the prescribed precision, applied to each product before accumulation.  The attention output is then computed by normalizing the accumulated sum:
\begin{equation}
\operatorname{Quant}_p\left(
\frac{\vs_j(u;T)}{D_j(u;T)}
\right).
\end{equation}
In particular, for disjoint sets $T_1,T_2\subseteq[u]$, exact accumulation ensures
\begin{equation}
\begin{aligned}
\vs_j(u;T_1\cup T_2)&=\vs_j(u;T_1)+\vs_j(u;T_2),\\
D_j(u;T_1\cup T_2)&=D_j(u;T_1)+D_j(u;T_2).
\end{aligned}
\label{eq:appendix_statistics_merge}
\end{equation}
Thus, independently computed partial statistics can be combined to reproduce the attention output exactly, as required by our communication protocols. This convention applies to both full and sparse attention throughout our proofs..

\paragraph{Common hidden width and recurrent memory.}
For simplicity, the main text uses a common width $d$ for token representations and recurrent states. An implementation may instead have token width $d_0$ and a larger persistent state. For example, sliding-window attention with a fixed window of $w$ predecessors can retain their keys and values in a state with $2wd_0$ coordinates, together with a constant amount of indexing information. Kernel-based linear attention with feature dimension $r$ retains a feature--value accumulator and a feature sum, with $r d_0+r$ coordinates for a single head~\citep{katharopoulos2020transformers}. For fixed $w$, or fixed $r$, these sizes are constant multiples of $d_0$, independent of sequence length. For the usual choice $r=d_0$ in a single head, however, linear attention has a state quadratic in $d_0$; this cost must also be counted.

To realize these implementations within the main-text definitions, choose $d$ large enough for every token representation and every layer's complete persistent state, and divisible by the head count. Pad smaller representations and states with zeros, and extend each map by restricting its input to the original coordinates and padding its output.

\section{Theorem proofs}
\label{app:theory}



\paragraph{Communication lower bound used in the proofs.}
We use the deterministic consequence of \citet[Theorem~3.2]{viola2009one}: for fixed $q\geq2$, the one-way $(q+1)$-player protocol in the main text needs $\Omega_q(n^{1/q})$ bits to compute $q$-function composition. To match its table formulation to tree pointer jumping, take $b=\lfloor n^{1/q}\rfloor$ and a $b$-ary tree with $q$ pointer choices and Boolean leaf labels. Index the nodes at each depth injectively in $[n]$, since there are at most $b^q\leq n$ of them. Encode the root choice by $x$, the remaining pointer maps by $f_q,\ldots,f_2$, and the leaf labels by $f_1$, using output symbols $1,2$. Assign fixed public values to unused table entries. This preserves each player's hidden input and the speaking order. A composition protocol therefore solves the tree problem. Announcing its decoded Boolean output adds one bit, reconciling the source's charged-output convention. The tree lower bound is $\Omega_q(b)$, giving the required bound.

\subsection{Proof of Theorem~\ref{thm:fixed_access_lower_bound}}
\label{app:fixed_access_proof}



\begin{proof}
Set $f_3,\ldots,f_m$ to the identity when $m>2$. On the resulting input of length $N=mn+1$, the required output is $f_1(f_2(x))$. We construct the three-player protocol of Definition~\ref{def:nof_protocol}: $P_3$ knows $f_1,f_2$ but not $x$; $P_2$ knows $f_1,x$ but not $f_2$; and $P_1$ knows $f_2,x$ but not $f_1$. The model, embedding, decoder,  All identity-table tokens, and task parameters are public. Messages are broadcast in the order $P_3\to P_2\to P_1$.

\paragraph{Player $P_3$: communicating boundary states.}
Write $\vh_t^{\mathrm{pre},\ell}$ and $\vh_t^{\mathrm{post},\ell}$ for the recurrent states after position $t$ in layer $\ell$ of the respective stacks.  $P_3$ is able to evaluate the entire block on positions $1,\ldots,N-1$, without knowing $x$. It broadcasts
\begin{equation}
\left(\vh_n^{\mathrm{pre},{L_{\mathrm{pre}}}},
      \vh_{N-1}^{\mathrm{pre},{L_{\mathrm{pre}}}}\right)
\qquad
\left(\vh_{N-1}^{\mathrm{post},\ell}\right)_{\ell\in[L_{\mathrm{post}}]}.
\label{eq:fixed_access_boundary_message}
\end{equation}

These messages require $pd(2+L_{\mathrm{post}})$ bits. The transmitted pre-attention states allow subsequent players who know $x$ to compute its pre-attention representation $\vz_N$, and players who know $f_2,x$ but not $f_1$ to reconstruct the pre-attention representations $\vz_{n+1},\ldots,\vz_N$. Moreover, once the attention output at position $N$ has been recovered,  the post-attention states at position $N-1$ allow the final player to complete the forward pass.

\paragraph{Player $P_2$: partial attention and candidate records.}
Using its knowledge of $f_1$,  player $P_2$ first computes the pre attention representations $\vz_t$, as well as  all key and value vectors for $t\in[n]$. From the boundary states at $N-1$ and its knowledge of $x$, it also computes $\vz_N$ and the final queries $\vq_N^j$ for every head. For each head $j$, it broadcasts
\begin{equation}
\vs_j(N;[n]),\qquad D_j(N;[n])
\label{eq:fixed_access_partial_attention}
\end{equation}
defined in equation~\ref{eq:appendix_attention_statistics}, costing $p(d+h)$ bits in total. It additionally selects
\begin{equation}
T_j=\operatorname{Top}_{j,\kappa}(N;[n])
\end{equation}
and transmits, for each $t\in T_j$, the record $(\vk_t^j,\vv_t^j, \alpha^j_{t,N})$ costing a total of $2p(d+h)\min\{\kappa,n\}$ bits.

\paragraph{Player $P_1$: reconstructing full attention and top-$\kappa$ access.}
Starting from the pre-attention boundary states $\vh_n^{\mathrm{pre},{L_{\mathrm{pre}}}}$, player $P_1$ evaluates the pre-attention representations $\vz_i$ for $n+1,\ldots,N$. This suffix consists of $f_2$, the public identity tables, and $x$. It also obtains  every suffix key and value together with the final queries  $\vq^j_N$ in the same way they were computed by $P_2$. For each head, it computes $\vs_j(N;\{n+1,\ldots,N\})$ and $D_j(N;\{n+1,\ldots,N\})$, adds the received prefix statistics, and normalizes with the model's rounding rule. Equation~\ref{eq:appendix_statistics_merge} guarantees that this is the exact full-attention head output, including the final token's self contribution. Writing $\vo_N^j\in\sR^{d_h}$ for the reconstructed head-output column vector, it obtains
\begin{equation}
\vx_N^{(0)}
=\vz_N+\mW_O^\top
\begin{bmatrix}\vo_N^1\\ \vdots\\ \vo_N^h\end{bmatrix}.
\label{eq:appendix_full_attention_reconstruction}
\end{equation}
Thus computing  the representation of the final token $x$ after the attention layer. 

For retrieval, it selects its own top-$\kappa$ candidates among $\{n+1,\ldots,N-1\}$ and merges them with $T_j$. We use the following elementary identity, with all rankings following the same total order:
\begin{equation}
\operatorname{Top}_{j,k}(u;A\cup C)
=
\operatorname{Top}_{j,k}\bigl(u;
\operatorname{Top}_{j,k}(u;A)\cup
\operatorname{Top}_{j,k}(u;C)\bigr)
\label{eq:topk_merge}
\end{equation}
for disjoint predecessor sets $A,C$. Indeed, an element excluded from its own set's top $k$ has at least $k$ elements of that set ranked ahead of it, and cannot be in the global top $k$. This proves that the merge recovers every selected predecessor, including when scores tie or one side has fewer than $k$ elements. Adding the known self position $N$ recovers $\mathcal N_j^{(\kappa)}(N)$ and all its keys and values.

Player $P_1$ now advances the final token through the post-attention stack. Starting with $\vx_N^{(0)}$, it knows at each layer the current token representation, that layer's state at $N-1$, and every entry of the fixed retrieved cache $\mC_N$. It therefore applies the specified transition and output maps exactly. Induction over these layers recovers the block's final representation, after which the public decoder outputs $f_1(f_2(x))$. No player needs the hidden table to execute its assigned computation.

\paragraph{Communication and the lower bound.} Since $h \leq d$ The total cost is at most
\begin{equation}
\begin{aligned}
&pd(L_{\mathrm{post}}+2)
+p(d+h)+2p(d+h)\min\{\kappa,n\}\\
&\hspace{1cm}\leq
pd(L_{\mathrm{post}}+2+4\kappa).
\end{aligned}
\label{eq:fixed_access_protocol_cost}
\end{equation}
The three-player NOF lower bound stated in the main text requires $\Omega(\sqrt n)$ bits for two-function composition. Applying it to this exact simulation completes the proof.
\end{proof}

\paragraph{Remark: comparison with prior bounds and tightness.}
Theorem~1.1 of \citet{ye2026hybridhierarchy} studies a similar setting, considering  multiple full-attention layers and a bounded number of recurrent layers between them. In their notation, taking $L=2$ leaves one full-attention layer, permits up to $2^{12}$ efficient layers after it, and gives a width--precision lower bound of  $N^{2^{-10}}=N^{1/1024}$.  Our $\Omega(\sqrt n)$ bound significantly improves this lower bound for the single hybrid block setting, applies to arbitrary post-attention recurrent depth with its state cost charged, and allows top-$\kappa$ cache access. 

We also highlight the role the of recurrent layers in the above bound. For a single attention layer without recurrence, the communication argument of \citet[Theorem~1]{peng2024limitations} shows a stronger lower bound of $\Omega(n\log n)$  for two-function composition.  The weaker $\Omega(\sqrt n)$ hybrid lower bound is consistent with RNNs providing an additional expressive advantage around a single attention layer. The tightness of our $\Omega(\sqrt{n})$ lower bound remains an open question.

\subsection{Proof of Theorem~\ref{thm:mura_main_upper}}
\label{app:mura_expressivity}


\begin{proof}
We construct an $m$-hop block $\hybrid_{\vkappa}$ with $\vkappa=(1,\ldots,1)$, $h=1$, $L_{\mathrm{pre}}=0$, $L_{\mathrm{post}}=1$, and $d=4m+7$ which solves $m$- function composition.

Fix $n\geq2$. Let $t_{r,i}=(r-1)n+i$ denote the position containing $f_r(i)$, and let $t_\star=mn+1=N$. In this proof, $\ve_1,\ldots,\ve_{m+1}$ are the standard basis of $\sR^{m+1}$, and we suppress the sole head index $j=1$ in queries, keys, values, and neighborhoods.

\paragraph{Embedding the function tables and query.}
We use the the following token embedding:
\begin{equation}
\te_n(f_r(i),t_{r,i})
=\begin{bmatrix}
1\\i\\i^2\\f_r(i)\\f_r(i)^2\\
\ve_r\\f_r(i)\ve_r\\\mathbf0_{2m}
\end{bmatrix},
\qquad
\te_n(x,t_\star)
=\begin{bmatrix}
1\\0\\0\\x\\x^2\\
\ve_{m+1}\\x\ve_{m+1}\\\mathbf0_{2m}
\end{bmatrix}.
\label{eq:mpj_function_embedding}
\end{equation}
Both vectors have $5+2(m+1)+2m=4m+7$ coordinates. These features depend only on the visible token and position; in particular, they do not reveal a function value at any other argument. The last $2m$ coordinates are reserved for sparse-attention output.

\paragraph{Full attention layer $\attn$.}
Let $\mS \in \sR^{(m+1)\times (m+1)}$ be a linear map such that $\mS\ve_1=\mathbf0$ and $\mS\ve_r=\ve_{r-1}$ for $2\leq r\leq m+1$. Choose $\mW_{Q}$ and $\mW_{K}$ so that their query and key vectors have the following coordinates, followed by zero padding to dimension $d$:
\begin{equation}
\vq_{t_{r,i}}=
\begin{bmatrix}2f_r(i)\\-1\\-f_r(i)^2\\\mS\ve_r\end{bmatrix},
\qquad
\vk_{t_{r,i}}=
\begin{bmatrix}i\\i^2\\1\\\ve_r\end{bmatrix},
\qquad
\vq_{t_\star}=
\begin{bmatrix}2x\\-1\\-x^2\\\mS \ve_{m+1}\end{bmatrix}.
\label{eq:mpj_query_key}
\end{equation}
The key of the final token is obtained by the same projection; its specific value is irrelevant to predecessor selection.

For $r\geq2$ and any causal predecessor $t_{s,b}<t_{r,i}$, the unscaled query--key score is
\begin{equation}
\langle\vq_{t_{r,i}},\vk_{t_{s,b}}\rangle
=-(b-f_r(i))^2+\mathbbm{1}\{s=r-1\}.
\label{eq:mpj_attention_score}
\end{equation}
The predecessor $t_{r-1,f_r(i)}$ has score $1$, whereas every other predecessor has score at most $0$. Likewise,
\begin{equation}
\langle\vq_{t_\star},\vk_{t_{s,b}}\rangle
=-(b-x)^2+\mathbbm{1}\{s=m\}
\end{equation}
uniquely selects $t_{m,x}$. The common positive scaling $1/\sqrt d$ preserves these rankings. Thus the top-$1$ neighborhoods determined by $\attn$ follow the required dependent lookups. Setting $c_0=x$ and $c_r=f_{m-r+1}(c_{r-1})$ for $r\in[m]$, it follows from the above that for every $r\in[m]$,
\begin{equation}
\mathcal N^{(\overbrace{1,\ldots,1}^{r})}(t_\star)
=\{t_\star\}\cup\{t_{m-s+1,c_{s-1}}:s\in[r]\}.
\label{eq:mpj_final_retrieval_set}
\end{equation}
Specifically, the index $t_{1,c_m}$ which contains the value of the full function composition, is contained in the above neighborhood.


Choose $\mW_{V}$ so that the full-attention value vectors are
\begin{equation}
\vv_{t_{r,i}}=
\begin{bmatrix}f_r(i)\ve_r\\\ve_r\end{bmatrix},
\qquad
\vv_{t_\star}=
\begin{bmatrix}x\ve_{m+1}\\\ve_{m+1}\end{bmatrix},
\label{eq:mpj_value_projection}
\end{equation}
padded to width $d$. For the embedded input matrix $\mX$, set $\mW_O=\mathbf0$, giving $\attn(\mX)=\mX$. The full-attention residual preserves every embedding coordinate, including the zero workspace.

\paragraph{\ourmethod attention layer  $\sattn$.}
Set $\mW^*_{Q}=\mathbf0$, so $\sattn$ attends uniformly to the $m+1$ retrieved positions using the shared values. From the retrieved value vectors, the only one whose first coordinate is not zero is $\vv_{t_{1, z_m}}$, whose first coordinate value is exactly $z_m = (f_1\circ\cdots\circ f_m)(x)$. Additionally, this is the only value vector whose $m+2$ entry is non zero, with value 1. after performing sparse attention, the value of the corresponding output vector at  these two coordinates are then 

\begin{equation}
\left(\frac{z_m}{m+1},\frac{1}{m+1}\right).
\label{eq:mpj_sparse_output}
\end{equation}

Choose $\mW^*_O$ the output projection matrix of $\sattn$ such that its only  nonzero entries are 
\begin{equation}
(\mW_O^{\mathrm{post},1})_{1,\,2m+8}=1,
\qquad
(\mW_O^{\mathrm{post},1})_{m+2,\,2m+9}=1.
\end{equation}
Thus the sparse layer writes $\frac{z_m}{m+1},\frac{1}{m+1}$ into workspace coordinates $2m+8,2m+9$. Their input coordinates are zero, so the sparse layer's residual adds no other contribution there.

For output symbol $a\in[n]$, let row $a$ of $\mU_n$ have coefficients $(2a,-a^2)$ on those two workspace coordinates and zero elsewhere. Its ideal logit is
\begin{equation}
2a\frac{z_m}{m+1}-a^2\frac{1}{m+1}
=\frac{2az_m-a^2}{m+1}
=\frac{z_m^2-(a-z_m)^2}{m+1}.
\label{eq:mpj_unembedding}
\end{equation}
The unique maximizer is $a=z_m$, with gap at least $1/(m+1)$. The linear decoder therefore extracts the composition value from an ordinary softmax-weighted average.

\paragraph{Finite precision.}
We show that the construction remains correct with $p=O(\log(mn))$ bits of precision per scalar. Choose a fixed-point format with dyadic grid spacing $\delta$ satisfying
\begin{equation}
\delta<
\min\left\{
\frac{1}{16\sqrt d},
\frac{1}{8(m+1)(n+1)^2},
\frac{1}{64N\sqrt d}
\right\},
\label{eq:mpj_precision_grid}
\end{equation}
taking the largest dyadic spacing satisfying this inequality. All unscaled query--key dot products are integers and can be computed exactly. We approximate the attention scaling factor $1/\sqrt d$ by a positive grid value $c$ with $|c-1/\sqrt d|\leq\delta$. Since each correct retrieval position has unscaled score $1$, while all competing positions have scores at most $0$, rounding the scaled logits preserves a gap of at least $c-2\delta>0$.

We choose $\operatorname{Exp}_p$ to be the exponential rounded to the fixed-point grid, with saturation to $[\delta,4]$ and $\operatorname{Exp}_p(0)=1$. The correct position consequently receives an unnormalized weight exceeding every competitor by at least $1/(2\sqrt d)$. Since the normalization factor is at most $4N$, the corresponding normalized attention scores differ by at least $1/(8N\sqrt d)$. The final constraint on $\delta$ ensures that this ordering is preserved even after quantization of the normalized scores. Thus, all top-$1$ retrievals remain unchanged.

In the sparse-attention layer, all queries are zero, so every attention logit is exactly zero. Choosing $\operatorname{Exp}_p(0)=1$ assigns equal weights to all $m+1$ retrieved positions. Since their value coordinates are integers, the unnormalized attention sums and normalization factor are computed exactly. Quantization of the normalized output introduces an error of at most $\delta$ in each workspace coordinate of~\eqref{eq:mpj_sparse_output}. Consequently, exact accumulation of the decoder's weighted sums, followed by final rounding, introduces an error of at most
\begin{equation}
(2n+n^2+1)\delta=(n+1)^2\delta
<\frac{1}{8(m+1)}
\end{equation}
in each output logit. Since the ideal logit gap is at least $1/(m+1)$, the correct symbol remains the unique maximizer, ensuring exact decoding.

Finally, we verify that all intermediate computations fit within the claimed precision. Since $\operatorname{Exp}_p$ takes values in $[\delta,4]$ and value coordinates have magnitude at most $n$, each quantized weighted-value coordinate has magnitude at most $4n+\delta$. Accumulating up to $N$ terms therefore requires representing values of magnitude $O(nN)=O(mn^2)$. All other intermediate quantities, including embeddings, routing products, and decoder products, have magnitude $O(n^2)$. Moreover, the choice of $\delta$ in~\eqref{eq:mpj_precision_grid} requires only $O(\log(mn))$ fractional bits. Thus, including integer bits and accumulation guard bits, $p=O(\log(mn))$ bits per scalar suffice. Since $d=4m+7$, $L_{\mathrm{pre}}=0$, and $L_{\mathrm{post}}=1$, the construction has capacity $\capacity(\hybrid_{\vkappa})=2pd=O(m\log(mn))$, which reduces to $O(\log N)$ for fixed $m$.
\end{proof}


\subsection{Proof of Theorem~\ref{thm:mura_main_lower}}
\label{app:mura_lower_proof}


\begin{proof}
Let $N=(m+1)n+1$, and assume $\hybrid_{\vkappa}$ is a $\vkappa$-\ourmethod block that is able to solve $q$-function composition, where $\vkappa \in \sN^m$. We generalize  the proof of Theorem~\ref{thm:fixed_access_lower_bound}, and show that there exists a $(m+2)$-player NOF protocol (Definition \ref{def:nof_protocol}) that is able to simulate the forward pass of $\hybrid_{\vkappa}$, with commmunication complexity $\capacity(\hybrid_{\vkappa})$.  The proof then follows from the classic $m+2$-player NOF protocol communication bound $\Omega({n^\frac{1}{m+1}})$ \cite{viola2009one}.

\paragraph{Player $P_{m+2}$: communicating boundary states.}

Write $\vh_t^{\mathrm{pre},\ell}$  for the recurrent states after position $t$ in layer $\ell$ of the pre attention recurrent layers. Similarly to the proof of Theorem \ref{thm:fixed_access_lower_bound},   $P_{m+2}$ is able to evaluate the entire block on positions $1,\ldots,N-1$, without knowing $x$. It broadcasts

\begin{equation}
\left(\vh_{rn}^{\mathrm{pre},{L_{\mathrm{pre}}}}\right)_{r\in [m+1]}.
\label{eq:mura_boundary_message}
\end{equation}
These messages require $pd(m+1)$ bits. The transmitted pre-attention states allow subsequent players who know $f_1,\dots,f_{r-1},f_{r+1},\dots, f_N,x$, but not $f_i$, to compute  pre-attention representations $\vz_{t}$, for $t =1,\dots n(r-1), nr+1, \dots, N$. 

The following communication steps are used to construct two main quantities needed for the last player $P_1$ to complete the forward pass, the full attention computation, and the retrieval set. For clarity, we present these two computations separately, but note that they happen in parallel, each player filling its role in each of the computations simultaniously.

\textbf{Full attention computation:}
Using its knowledge of $f_1,\dots, f_m$, player $P_{m+1}$ first computes the pre attention representations $\vz_t$, as well as  all key and value vectors for $t\in[mn]$. From the boundary states at $N-1$ and its knowledge of $x$, it also computes $\vz_N$ and the final queries $\vq_N^j$ for every head. For each head $j$, it broadcasts
\begin{equation}
\vs_j(N;[mn]),\qquad D_j(N;[mn])
\label{eq:mura_partial_attention}
\end{equation}

Starting from the pre-attention boundary states $\vh_{mn}^{\mathrm{pre},{L_{\mathrm{pre}}}}$, player $P_m$ evaluates the pre-attention representations $\vz_i$ for $mn+1,\ldots,N$. This suffix consists of $f_{m+1}$ and $x$. It then computes   every key and value vector from this suffix, together with the final queries  $\vq^j_N$. For each head, it computes $\vs_j(N;\{nm+1,\ldots,N\})$ and $D_j(N;\{nm+1,\ldots,N\})$, adds the received prefix statistics, and normalizes with the model's rounding rule. Equation~\ref{eq:appendix_statistics_merge} guarantees that this is the exact full-attention head output, including the final token's self contribution. Writing $\vo_N^j\in\sR^{d_h}$ for the reconstructed head-output column vector, it transmits
\begin{equation}
\vx_N^{(0)}
=\vz_N+\mW_O^\top
\begin{bmatrix}\vo_N^1\\ \vdots\\ \vo_N^h\end{bmatrix}.
\label{eq:mura_full_attention_reconstruction}
\end{equation}
Thus computing  the representation of the final token $x$ after the attention layer. The total communication cost of this process is $p(d+h)$.

\textbf{Retrieved KV cache computation:}


For a selected position $t$ in head $j$, the record $R_j(t)$ is defined by
\begin{equation}
R_j(t)=(t,\vq_t^j,\vk_t^j,\vv_t^j).
\label{eq:mura_record_size}
\end{equation}
The cost of communicating a single record is 
\begin{equation}
    b_{\mathrm{rec}}=3p(d/h+1).
\end{equation}
 Additionally, define recursively for $\quad r\in[m]$,
\begin{equation}
\begin{aligned}
F_{j,0} &= G_{j,0} = \{N\},\\
F_{j,r} &= \bigl\{ t \in \operatorname{Top}_{j,\kappa_r}(s;[N])
\mid s \in F_{j,r-1} \bigr\},\\
G_{j,r} &= \Bigl\{ t \in \operatorname{Top}_{j,\kappa_r}
\bigl(s;[1,\dots,nr] \cup [n(r+1)+1,\dots,N]\bigr)|\, s \in F_{j,r-1} \Bigr\}.
\end{aligned}
\label{eq:mura_frontier}
\end{equation}

Here, $F_{j,r}$ denotes the set of positions reached from the final token $x$ after exactly $r$ successive retrieval hops with widths $(\kappa_1,\ldots,\kappa_r)$. The auxiliary set $G_{j,r}$ contains the candidates obtained at the $r$-th hop when positions corresponding to the function table $f_r$ are excluded from retrieval.

We prove inductively that each player $P_i$, for $i\in[m+1]$, can recover the sets $F_{j,m+1-i}$ and $G_{j,m+2-i}$, together with the corresponding token records, for every attention head $j\in[h]$ (with $G_{j,m-i}$ considered only for $i\leq m$). This ensures that the final player $P_1$ can reconstruct the retrieved KV entries corresponding to all sets $F_{j,r}$, for $r\in[m]$ and $j\in[h]$. Since $P_1$ can also reconstruct the full-attention output, as shown above, and knows the query $x$, it can complete the remaining forward pass and produce the correct output, completing the proof.

The base case $i=m+1$ is simple. As shown above, player $P_{m+1}$ can compute the final pre-attention representation $\vz_N$ and hence the record $R_j(N)$ for every head $j\in[h]$. Moreover, it can compute the representations $\vz_t$ for all $t \in \operatorname{Top}_{j,\kappa_r}
\bigl(s;[1,\dots,nr] \cup [n(r+1)+1,\dots,N]\bigr)$ eligible for selection into $G_{j,1}$, allowing it to compute the top $\kappa_1$ tokens and transmit their corresponding records.

For the inductive step, suppose player $P_i$ receives from $P_{i+1}$ the sets $F_{j,m-i}$ and $G_{j,m-i+1}$, together with their corresponding token records. Write $r=m+1-i$, and let $S_r = [nr+1,\dots, n(r+1)]$ denote the positions of the function table excluded in the definition of $G_{j,r}$. To recover $F_{j,r}$, player $P_i$ computes, for each $s\in F_{j,r-1}$, the top-$\kappa_r$ candidates within $B_r$. Since $P_i$ knows the corresponding function table and has access to the necessary recurrent boundary states, it can reconstruct these tokens' representations and attention scores. Combining these candidates with those in $G_{j,r}$ yields the full set $F_{j,r}$, by the identity
\begin{equation}
\operatorname{Top}_{j,k}(s;A\cup B)
=
\operatorname{Top}_{j,k}\bigl(s;
\operatorname{Top}_{j,k}(s;A)
\cup
\operatorname{Top}_{j,k}(s;B)\bigr),
\end{equation}
which holds for disjoint sets $A,B$ under consistent tie-breaking.

Having recovered $F_{j,r}$ and its corresponding records, player $P_i$ computes $G_{j,r+1}$ by selecting the top-$\kappa_{r+1}$ candidates for each $s\in F_{j,r}$, excluding the function table unavailable to $P_i$. Since all remaining tokens are known to this player, their representations and attention scores can be computed using the transmitted boundary states. It then sends $F_{j,r}$ and $G_{j,r+1}$, together with their records, to $P_{i-1}$. This completes the inductive step.

At each retrieval hop $r\in[m]$, the transmitted records require at most $3p(d+h)\prod_{\ell=1}^{r}\kappa_\ell$ bits. Since $\kappa_\ell\geq1$ for all $\ell$, including the initial messages gives a total communication cost of at most
\begin{equation}
dp(m+1)+(d+h)p+3p(d+h)m\prod_{\ell=1}^{m}\kappa_\ell
=O(\capacity(\hybrid_{\vkappa})),
\end{equation}
where the last equality holds, using $h\leq d$. The classical $\Omega(n^{1/(m+1)})$ communication lower bound for one-way $(m+2)$-player NOF protocols computing $(m+1)$-function composition therefore implies
$\capacity(\hybrid_{\vkappa})=\Omega(n^{1/(m+1)})$, completing the proof.

\end{proof}

\section{Experimental details}
\label{app:experimental_details}

We provide the implementation and training details of our synthetic experiments below. All experiments were conducted on NVIDIA L40 GPUs using Python 3.11.16, PyTorch 2.12.1+cu126, CUDA 12.6, and NumPy 2.4.6.


\subsection{Dataset and Causal Supervision}
\label{app:dataset}

\paragraph{Sampling and serialization.}
Each example consists of five functions $f_1,\ldots,f_5:[n]\to[n]$, where $n$ is sampled from $\{25,\ldots,30\}$. Conditional on $n$, all $5n$ function values are sampled independently and uniformly from $[n]$. The functions are serialized as
\begin{equation}
\bigl(\mathrm{BOS},f_1(1),\ldots,f_1(n),\ldots,
f_5(1),\ldots,f_5(n)\bigr).
\label{eq:experimental_sequence}
\end{equation}
Unlike Definition~\ref{def:function_composition}, the sequence begins with a beginning-of-sequence (BOS) token and contains no explicit final query $x$. Function tables appear in increasing order, without separator tokens, yielding a sequence length of $N=5n+1\in[126,151]$.

The dataset contains 250,000 independently generated examples, divided into 200,000 training, 25,000 validation, and 25,000 test examples. Each split is stratified by $n$, with the number of examples for any two domain sizes differing by at most one. During batching, each example occupies a separate row, and shorter sequences are right-padded. Padding positions are excluded from attention, retrieval, loss computation, and accuracy evaluation.

\paragraph{Targets and loss.}
We evaluate the ability of hybrid blocks to compute function compositions of depths $r\in[4]$ using causal supervision at multiple positions. Specifically, each token containing $f_i(j)$, for $i\in\{2,\ldots,5\}$, is assigned $i-1$ targets, corresponding to compositions of depths $r\in[i-1]$:
\begin{equation}
z_{i,j,r}
=(f_{i-r}\circ\cdots\circ f_{i-1})(f_i(j)).
\label{eq:experimental_targets}
\end{equation}
Thus, tokens from the second function table predict $f_1(f_2(j))$, while tokens from later tables are supervised at progressively greater depths. The BOS token and the first function table receive no direct supervision.

Each depth $r\in[4]$ has a separate linear readout applied to the corresponding token representations, producing 30 output logits. For examples with $n<30$, logits corresponding to symbols outside $[n]$ are masked, and each token is evaluated only at depths for which a target is defined. Consequently, each example contributes $n(5-r)$ targets at depth $r$, for a total of $10n$ supervised predictions.

All valid predictions contribute equally to the training objective. Denoting by $\mathcal S_{\mathcal B}$ the set of supervised example--position--target triples in a logical batch $\mathcal B$, we minimize
\begin{equation}
\mathcal L_{\mathcal B}
=-\frac{1}{|\mathcal S_{\mathcal B}|}
\sum_{(e,p,t)\in\mathcal S_{\mathcal B}}
\ln p_r(t\mid e,p),
\label{eq:experimental_loss}
\end{equation}
where $p_r$ is the masked softmax distribution produced by readout $r$, and $t$ is its target at the corresponding example, position and composition depth $r$.


\subsection{Hybrid Block Architecture}
\label{app:architecture}

\paragraph{Token embedding map.}
All models use a token embedding scheme inspired by~\citet{kozachinskiy2025strassen}, consisting of nine deterministic features followed by a learned affine projection to width $d=16$. With the BOS token at position $1$, each subsequent token at position $t\in\{2,\ldots,N\}$ is characterized by its function argument $i=1+((t-2)\bmod n)$, function-table index $b=1+\lfloor(t-2)/n\rfloor$, and token value $v\in[n]$.

Define the circular encoding
\begin{equation}
\vc_R(a)=
\left(\cos\frac{2\pi(a-1)}{R},
      \sin\frac{2\pi(a-1)}{R}\right)^\top,
\qquad R_n=n+1.
\end{equation}
The nine input features are given by
\begin{equation}
\begin{aligned}
\bm\phi(t,v)=\bigl[&
R_n\vc_{R_n}(v)^\top,\ R_n\vc_{R_n}(i)^\top,\\
&5\vc_5(b+1)^\top,\ 5\vc_5(b)^\top,\ 1
\bigr]^\top.
\end{aligned}
\label{eq:experimental_features}
\end{equation}
The two function-table encodings use offsets $b+1$ and $b$, respectively. The BOS token is assigned the feature vector $(0,0,0,0,5R_n,0,0,0,1)^\top$, while padding tokens receive zero features. All feature vectors are mapped to $\sR^{16}$ by a learned affine transformation, with padding positions masked again afterward to eliminate the effect of the bias.

This encoding provides the function-table index, argument, and value of each token without revealing composition paths or targets. We use no additional learned token or function embeddings, absolute positional embeddings, or rotary positional transformations.

\paragraph{Layers.}
Each hybrid block consists of one full-attention layer followed by four efficient sequence layers, implemented as recurrent or sparse-attention layers depending on the architecture. All layers have width $d=16$, and use four heads. As described in Appendix~\ref{app:preliminaries}, each layer uses the pre-normalization and position-wise feed-forward structure of Equation~\ref{eq:appendix_transformer_wrapper}. We use affine LayerNorm with $\epsilon=10^{-5}$ and apply an additional affine LayerNorm before the output readouts. Similarly to \cite{kozachinskiy2025strassen}, each of the five feed-forward networks has two hidden layers of width $64$, with SwiGLU activations, giving dimensions $16\to64\to64\to16$. Both hidden layers use separate gate and value projections. All model parameters use standard initialization.

\paragraph{\ourmethod retrieval.}
\label{app:retrieval}
To evaluate the effect of retrieval depth, we consider four \ourmethod configurations with retrieval widths $\vkappa=(32)$, $(16,1)$, $(11,1,1)$, and $(8,1,1,1)$, corresponding to depths one through four. These configurations have  matched retrieval budgets of 32 tokens across all hops, except for $(11,1,1)$, which retrieves up to 33 tokens.

Each \ourmethod layer has four attention heads, with retrieval depths assigned to encourage learning compositions of different orders. Specifically, for a configuration $\vkappa=(\kappa_1,\ldots,\kappa_m)$, head $j\in[4]$ uses retrieval widths $(\kappa_1,\ldots,\kappa_{\min(j,m)})$ and attends only to positions reached after exactly $\min(j,m)$ hops. Thus, successive heads perform increasingly deep retrievals, up to the maximum depth $m$. This assignment is repeated across all four \ourmethod layers. Routing decisions are detached from the computational graph, while gradients continue to flow through the selected keys and values.

\paragraph{Hybrid block baselines.}
\label{app:baselines}
We compare \ourmethod against four hybrid block baselines: kernel-based linear attention~\citep{katharopoulos2020transformers}, sliding-window attention~\citep{beltagy2020longformer}, Gated DeltaNet~\citep{yang2024gated}, and a simplified HySparse variant~\citep{gao2026hysparse}, equivalent to one-hop \ourmethod with $\vkappa=(32)$. For a fair comparison, all architectures use the same width $d=16$, one full-attention layer, and four subsequent efficient sequence layers, with four attention heads where applicable.

For sliding-window attention, we evaluate two window sizes, 32 and 16. The former matches the retrieval budget of \ourmethod but incurs greater memory and computational overhead, since each of the four sliding-window layers maintains its own KV cache. In our short-sequence setting, these caches collectively store approximately as many KV entries as an additional full-attention layer. We therefore also evaluate a smaller window of size 16, providing a more computationally efficient baseline.

\subsection{Optimization and Training Schedules}
\label{app:optimization}

All models are trained for 500 epochs using AdamW with a batch size of 2,048, weight decay of 0.01, and optimizer parameters $\beta_1=0.9$, $\beta_2=0.999$, and $\epsilon=10^{-8}$. The learning rate is linearly warmed up over the first 40 epochs to a peak of $10^{-3}$, followed by cosine decay to one-tenth of the peak rate. After gradient accumulation, the global gradient $\ell_2$ norm is clipped to 1.0 before each optimization step.

\section{Extended related work}
\label{app:extended_related_work}

\paragraph{Efficient sequence models and shared attention memory.}
Reducing the cost of full attention has motivated a range of efficient sequence architectures. Early approaches introduced sparse attention mechanisms~\citep{child2019generating,beltagy2020longformer,zaheer2020big}; subsequent work developed recurrent and state-space architectures, including S4~\citep{gu2021efficiently}, Mamba~\citep{gu2023mamba}, and Gated DeltaNet~\citep{yang2024gated}. These recurrent architectures support parallel training algorithms and maintain fixed-size states during autoregressive inference. Hybrid architectures combine efficient layers with periodic full attention, as in Kimi 3~\citep{team2026kimi}, Gemma 3~\citep{gemmateam2025gemma3}, Jamba~\citep{lieber2024jamba}, and OLMo Hybrid~\citep{merrill2026olmohybrid}. Their efficient layers typically maintain memory independently of the full-attention KV cache. Recent work has explored sharing this memory. Cross-Layer Attention~\citep{brandon2024crosslayer} shares KV activations across attention layers, while DART~\citep{qian2026dart} attends to compressed recurrent chunk states. These approaches reduce memory requirements, although each attention operation still accesses a number of entries that grows with sequence length. HySparse and HySparse2~\citep{gao2026hysparse,hysparse2} instead reuse a full-attention KV cache through bounded sparse retrieval in subsequent layers. Graph Machine~\citep{hou2026graphmachine}, which is closely related to our work, also uses multi-hop sparse retrieval in hybrid architectures. It constructs and updates learned weighted edges across layers and operates on layer-specific representations. In \ourmethod, retrieval traverses a fixed attention graph; the subsequent sparse layers reuse the resulting neighborhoods and the original full-attention KV cache.

\paragraph{Expressivity as a guide to architecture design.}
Prior work has examined sequence-model expressivity from several perspectives. Circuit-complexity analyses establish limitations of log-precision Transformers~\citep{merrill2023parallel} and study how padding tokens\citep{merrill2026exact}, recurrent application of layers~\citep{dehghani2018universal} and chain of thought tokens \citep{merrill2024expressive} expand their capabilities. Other studies analyze the resources required for tasks such as sparse averaging, 3-token Matching, and function composition~\citep{sanford2023representational,sanford2024parallel,peng2024limitations,chen2024theoretical,kozachinskiy2025strassen}. Limitations identified for shallow Transformers have also motivated mechanisms such as Strassen attention~\citep{kozachinskiy2025strassen}. Complementarily, \citet{gelberg2026compressibility} study the relationship between Transformer representations and KV-cache compressibility, relating expressive capacity and memory requirements to the learned representations.

Beyond Transformers, analyses of linear RNNs identify limitations in state tracking~\citep{merrill2024illusion,sarrof2024expressive}, while \citet{grazzi2025unlocking} show that allowing negative transition eigenvalues can improve these capabilities. \citet{merrill2026olmohybrid} establish expressive advantages of combining attention and recurrence, and connect these results to the design of a large-scale hybrid model. Closely related to our analysis, \citet{ye2026hybridhierarchy} prove lower bounds for hybrid architectures on sequential function composition. We study a simpler composition task and obtain a stronger capacity lower bound for a single full-attention block, allowing arbitrary recurrent depth with its state budget accounted for as in Appendix~\ref{app:theory}, as well as top-$\kappa$ shared-memory access. Their framework also covers architectures with multiple full-attention layers. The comparison and the open question of tightness are discussed after the proof in Appendix~\ref{app:fixed_access_proof}.

Expressivity analysis also guides architecture design for other structured data. Examples include graph neural networks~\citep{xu2018powerful,morris2019weisfeiler,barshalom2024coarsening,eitan2025derivatives,southern2025hymn}, topological neural networks~\citep{bodnar2021weisfeiler,eitan2025topological}, weight- and gradient-space architectures~\citep{navon2023equivariant,dayan2026weightspace,gelberg2025gradmetanet}, and permutation-invariant models for sets and sets of structured elements~\citep{zaheer2017deep,maron2020learning}. These studies use formal expressivity results to inform architectural choices that respect the structure and symmetries of their inputs.

\end{document}